\documentclass{article}
\usepackage[utf8]{inputenc}
\usepackage[T1]{fontenc}
\usepackage{graphicx}
\usepackage{subcaption}
\usepackage{mathtools}
\usepackage{booktabs}
\usepackage{bm}
\usepackage{tcolorbox}
\definecolor{emerald}{rgb}{0.31, 0.78, 0.47}
\definecolor{carminepink}{rgb}{0.92, 0.3, 0.26}
\usepackage{setspace}
\usepackage{tikz}
\usetikzlibrary{shapes,backgrounds}
\definecolor{dodgerblue}{rgb}{0.12, 0.56, 1.0}
\definecolor{ceruleanblue}{rgb}{0.16, 0.32, 0.75}
\definecolor{bananayellow}{rgb}{1.0, 0.88, 0.21}
\definecolor{arylideyellow}{rgb}{0.91, 0.84, 0.42}
\definecolor{beaublue}{rgb}{0.74, 0.83, 0.9}
\definecolor{babyblueeyes}{rgb}{0.63, 0.79, 0.95}
\definecolor{ceil}{rgb}{0.57, 0.63, 0.81}
\usepackage{multicol,multirow}
\usepackage{natbib}
\usepackage{comment}
\newcommand{\cu}{\mathcal{U}}
\newcommand{\blockedit}{\color{black}}

\usepackage{booktabs}
\newcommand{\ie}{\textit{i.e.}}
\newcommand{\eg}{\textit{e.g.}}

\newcommand{\Eg}{\textit{E.g.}}
\usepackage{enumerate}   
\usepackage{amsmath}
\usepackage{multirow}
\usepackage{bbm}
\usepackage{tikz}

\definecolor{mygreen}{rgb}{0,0.6,0}
\usetikzlibrary{positioning,arrows,decorations,decorations.markings,shadows,positioning,arrows.meta,matrix,fit}

\usetikzlibrary{positioning,arrows}
\usepackage{wrapfig}

\DeclareMathOperator*{\argmin}{arg\,min}

\usepackage{amsmath}
\usepackage{appendix}
\usepackage{amssymb}

\usepackage{amsthm}
\usepackage{hyperref}
\usepackage[capitalise]{cleveref}
\Crefname{assumption}{Assumption}{Assumptions}

\theoremstyle{plain}
\newtheorem{theorem}{Theorem}

\newtheorem{corollary}{Corollary}
\theoremstyle{definition}
\newtheorem{assumption}{Assumption}

\newtheorem{remark}{Remark}

\usepackage{graphicx}

\newcommand{\frs}{\mathfrak{S}}

\def\Cramer{Cram\'{e}r}

\newcommand{\cm}{\mathcal{M}}

\newcommand{\bigO}{\mathcal{O}} %

\renewcommand{\P}{\mathbb{P}}

\newcommand{\op}{\mathrm{o}_{p}}
\newcommand{\E}{\mathbb{E}}
\newcommand{\G}{\mathbb{G}}

\newcommand{\pa}{\mathrm{\pa}}

\newcommand{\var}{\mathrm{var}}

\newcommand{\rd}{\mathrm{d}}

\newcommand{\braces}[1]{\left\{#1\right\}}
\newcommand{\bracks}[1]{\left[#1\right]}

\newcommand{\abs}[1]{\left|#1\right|}

\newcommand{\R}[1]{\mathbb{R}^{#1}}
\newcommand{\epol}{\pi^\mathrm{e}}

\newcommand{\DR}{\mathrm{DR}}

\newcommand{\dr}{\mathrm{DR}}
\newcommand{\drpihat}{\dr\text{-$\hat\pi_*$}}

\newcommand{\ips}{\text{IS-Avg}}
\newcommand{\bal}{\text{IS}}
\newcommand{\dett}{\text{det}}
\newcommand{\wei}{\text{\text{IS-PW}}}
\newcommand{\bi}{\text{BI}}

\newcommand{\mo}{\text{SMRDR}}
\newcommand{\mmo}{\mathrm{MRDR}}

\newcommand{\Scal}{\mathcal{S}}
\newcommand{\Acal}{\mathcal{A}}
\newcommand{\Ncal}{\mathcal{N}}

\newcommand{\Qcal}{\mathcal{Q}}

\newcommand{\Dcal}{\mathcal{D}}

\newcommand{\Rcal}{\mathcal{R}}
\newcommand{\Lcal}{\mathcal{L}}

\renewcommand{\eqref}[1]{(\ref{#1})}

\newcommand{\RN}[1]{%
  \textup{\uppercase\expandafter{\romannumeral#1}}%
}

\usepackage{color,graphicx,lscape,longtable}

\def\boxit#1{\vbox{\hrule\hbox{\vrule\kern6pt\vbox{\kern6pt#1\kern6pt}\kern6pt\vrule}\hrule}}

\usepackage{xcolor}
\newcount\Comments  
\newcommand{\kibitz}[2]{\ifnum\Comments=1\textcolor{#1}{#2}\fi}

\usepackage{thmtools}
\usepackage{thm-restate}

\usepackage{algorithmic}
\usepackage{algorithm}
\usepackage{comment}
\usepackage[utf8]{inputenc}
\usepackage{enumitem}

\usepackage{microtype}
\usepackage{graphicx}

\allowdisplaybreaks

\title{Optimal Off-Policy Evaluation from \\ Multiple Logging Policies }
\author{Nathan Kallus \\ Department of Operations Research and Information Engineering\\  and Cornell Tech
       Cornell University 
       \and 
         Yuta Saito \\ Department of Industrial Engineering and Economics \\
      Tokyo Institute of Technology 
      \and
       Masatoshi Uehara \footnote{Corresponding author \href{mailto:mu223@cornell.edu}{mu223@cornell.edu}} \\ Department of Computer Science and Cornell Tech\\
       Cornell University 
       }

\date{}

\begin{document}

\maketitle

\begin{abstract}
We study off-policy evaluation (OPE) from multiple logging policies, each generating a dataset of fixed size, \ie, stratified sampling. Previous work noted that in this setting the ordering of the variances of different importance sampling estimators is instance-dependent, which brings up a dilemma as to which importance sampling weights to use. In this paper, we resolve this dilemma by finding the OPE estimator for multiple loggers with \emph{minimum} variance for any instance, \ie, the \emph{efficient} one. In particular, we establish the efficiency bound under stratified sampling and propose an estimator achieving this bound when given consistent $q$-estimates. To guard against misspecification of $q$-functions, we also provide a way to choose the control variate in a hypothesis class to minimize variance. Extensive experiments demonstrate the benefits of our methods' efficiently leveraging of the stratified sampling of off-policy data from multiple loggers. 
\end{abstract}

\section{Introduction}

In many applications where personalized and dynamic decision making is of interest, exploration is costly, risky, unethical, or otherwise infeasible ruling out the use of online algorithms for contextual bandits (CB) and reinforcement learning (RL) that need to explore in order to learn. This includes both healthcare, where we fear bad patient outcomes, and e-commerce, where we fear alienating users.
This motivates the study of off-policy evaluation (OPE), which is the task of estimating the value of a given policy using only historical data, which is generated by current decision policies. This can support performance evaluation of policies with respect to various rewards objectives in order to better understand their behavior before deploying them in a real environment.
Given how invaluable this is, OPE has been studied extensively both in CB~\citep{kallus2018balanced,narita2018,wang2017optimal,dudik2014doubly,NIPS2017_6954,Muandet2018,SuYi2020AESf} and in RL~\citep{Chow2018,Liu2018,KallusNathan2019EBtC,Kallus2019IntrinsicallyES,KallusUehara2019,Munos2016,jiang,thomas2016,YinMing2020NOPU} and has been applied in various domains including healthcare~\citep{MurphyS.A.2003Odtr} and education~\citep{Mandel2014}.


In most of the above studies, the observations used to evaluate a new policy are assumed generated by a \textit{single} logging policy. Often, however, we have the opportunity to leverage multiple datasets, each potentially generated by a different logging policy \citep{AgarwalAman2017EEUL,HeLi2019OLfM,Strehl2010nips,Bareinboim2016}. The size of each dataset is generally fixed by design, which distinguishes this setting from a single logging policy given by the mixture of logging policies. Such fixed dataset sizes is an example of \emph{stratified sampling}~\citep{WooldridgeJeffreyM.2001APOW}, where the identity of the logging policies constitute the stratum.
%

The distinction of these two settings is crucial since the same estimator may have varying precision in each setting (a fact well-known in Monte Carlo integration, \citealp{geyer1994,A.Kong2003AToS,TanZhiqiang2004OaLA}, and noise contrastive estimation, \citealp{pmlr-v9-gutmann10a,Uehara}). Thus, many results in the standard \emph{un}stratified OPE setting cannot be directly translated to a multiple logger setting, most crucially the efficiency lower bound on mean-squared error (MSE) and estimators that achieve this lower bound \citep{narita2018,KallusUehara2019,dudik2014doubly,jiang}. In the multiple logger setting, we may additionally consider a much greater variety of estimators that can utilize the logger identity as data.
In this paper, we study a wide range of such estimators, establish the efficiency lower bound, and propose estimators that achieve it.


Previous work on OPE with multiple loggers proposed various importance sampling (IS) estimators that use the logger identity \citep{AgarwalAman2017EEUL}. However, they arrived at a \emph{dilemma}: there is no strict ordering between the IS estimate with marginalized logging probabilities and a precision-weighted combination of the IS estimates in each dataset. That is, which estimate has lower MSE depends on the problem instance and is not known a priori, and therefore it is not clear which should be preferred. 
Our analysis resolves this dilemma by developing an \emph{efficient} estimator, which has MSE better (or not worse) than both of the above.

Our contributions are as follows. First, when the logging policies are known, we study the variances of a new class of unbiased estimators that includes and is much bigger than the class considered in \citet{AgarwalAman2017EEUL}. This new class incorporates both control variates and flexible weights that may depend on logger identity. We show that a single estimator has minimum variance in this class (\cref{sec: infeasible unbiased,sec:feasible}). We extend this finite-sample bound to also bound the asymptotic MSEs of \emph{all} regular estimators, thereby establishing the efficiency lower bound (\cref{sec: regular est}). We show how to construct an efficient estimator even if behavior policies are unknown and establish theoretical guarantees for it (\cref{sec:efficient_estimation}). Then, we theoretically investigate the differences between OPE in the stratified and unstratified cases by showing that the variances of the estimator are generally different under two settings and are asymptotically equivalent \emph{only} when the estimator is efficient (\cref{sec:multi}). We use this insight to choose optimal control variates to directly minimize variance, extending the More Robust Doubly Robust (MRDR) estimator of \citet{rubin08,Chow2018} to the stratified setting (\cref{sec:more}). Finally, we study our new OPE methods empirically and compare them to benchmark methods including those of~\citet{AgarwalAman2017EEUL}. 

\section{Background}\label{sec:preli}

We start by setting up the problem  and summarizing the relevant literature. 

\subsection{Problem Setup}

We focus on the CB setting as was the topic of previous work \citep{AgarwalAman2017EEUL} and discuss the extension to RL in \cref{sec:extensions}.

We are concerned with the average reward of taking an action $a\in\Acal$ in context (state) $s\in \Scal$ when following the policy $\pi^e(a\mid s)$, known as the evaluation policy.
Both $\Acal$ and $\Scal$ may be discrete or continuous. 
Rewards $r\in[0,R_{\max}]$ are described by the (unknown) reward emission probability distribution $p_{R\mid S,A}(r\mid s,a)$,
and contexts are drawn from the (unknown) distribution $p_S(s)$. 
Thus, the average reward under $\pi^e$, which is our target estimand is
$$ 
J:=\E_{\pi_e}[r],
$$
where the subscript $\pi_e$ refers to the joint distribution $p_S(s)\pi^e(a\mid s)p_{R\mid S,A}(r\mid s,a)$ over $(s,a,r)$.

To help estimate $J$,
we consider observing $K$ datasets, $\Dcal=\{\Dcal_1,\cdots,\Dcal_K\}$, each of (fixed) size $n_k$ and associated with the logging policy $\pi_k(a\mid s)$, for $k\in[K]=\{1,\dots,K\}$. (We consider both the cases where $\pi_k$ are known and unknown.) Each dataset consists of observations of state-action-reward triplets, $\Dcal_k= \{(S_{kj},A_{kj},R_{kj})\}_{j=1}^{n_k}$, drawn independently according to the product distribution 
$$
(S_{kj},A_{kj},R_{kj})\sim p_S(s)\pi_k(a\mid s)p_{R\mid S,A}(r\mid s,a).
$$
Notice that the distribution above differs from the distribution in the definition of $J$ in the policy used to generate actions.
We let $n=n_1+\dots+n_K$ be the total dataset size.
We often reindex the whole data as $\Dcal=\bigcup_{k=1}^K\{(k,s,a,r):(s,a,r)\in\Dcal_k\}=\{(k_i,S_i,A_i,R_i):i=1,\dots,n\}$, treating the logger identity $k_i$ as an additional component of an observation in one big pooled dataset.
For a function $f(s,a,r)$ we let $\E_{n_k}[f]=\frac1{n_k}\sum_{(s,a,r)\in\Dcal_k}f(s,a,r)$ and for a function $f(k,s,a,r)$ we let $\E_n[f]=\frac1n\sum_{(k,s,a,r)\in\Dcal}f(k,s,a,r)$. As mentioned above, we let $\E_\pi$ refer to expectations with respect to the distribution on $(s,a,r)$ induced by playing $\pi$ (similarly, $\var_\pi$). Unsubscripted expectations and variances are with respect to the data generation (such as the variance of an estimator).

We let $\rho_k=n_k/n$ be the dataset proportions and $\pi_{*}(a\mid s)=\sum_{k=1}^K\rho_k\pi_k(a\mid s)$ be the marginal logging policy (as a policy, it corresponds to randomizing the choice of logger with weights $\rho_k$ and then playing the chosen logger, but note this is \emph{not} how the data is generated, as $n_k$ are fixed).
For any function $f(s,a)$, let $f(s,\pi)=\E_{\pi}[f(s,a)\mid s]=\int f(s,a)\rd\pi(a\mid s)$.
We let $q(s,a)=\E_{p_{R\mid S,A}}[r\mid s,a],~v(s)=q(s,\epol),\,\sigma_r^2(s,a)=\var_{p_{R\mid S,A}}[r\mid s,a]$.
We define the $L_2$ norm by $\|f\|_2=\{\E_{\pi_{*}}[f^2(s,a,r)]\}^{1/2}$. We denote the normal distribution with mean $\mu$ and variance $\sigma^2$ by $\Ncal(\mu,\sigma^2)$. 

We always let $n,n_1,\dots,n_K$ be fixed and finite.
When we discuss asymptotic behavior we consider sample sizes $n'=mn,n'_k=mn_k$ and $m\to\infty$ such that sample proportions $\rho_k=n_k/n=n'_k/n'$ remain fixed.

\subsection{Previous Work and the Multiple Logger Dilemma}

In the \emph{un}stratified setting, wherein the logging policy first chooses $k$ at random from $[K]$ with weights $\rho_k$ and then plays the logging policy $\pi_k$, the standard IS estimator would be 
$$
\hat J_{\text{IS}}:= \E_n\left[\frac{\epol(a\mid s)r}{\pi_{*}(a\mid s)}\right].
$$
This estimator can still be applied in the stratified setting in the sense that is unbiased under a weak overlap. 
\begin{assumption}[Weak Overlap]\label{asm:weak}
For any $s\in\Scal$,\break $\epol(\cdot\mid s)\ll \pi_*(\cdot\mid s)$ (where $\ll$ means absolutely continuous).
When $\abs{\Acal}<\infty$, this is equivalent to:
for any $s\in\Scal$, $a\in\Acal$, $\epol(a\mid s)>0$ implies $\pi_*(a\mid s)>0$.
\end{assumption}
\citet{AgarwalAman2017EEUL} study the multiple logger setting and propose estimators that combine the IS estimators in each of the $K$ datasets: given simplex weights $\lambda\in\Delta^K=\{\lambda\in\R K:\lambda_k\geq0,\sum_{k=1}^K\lambda_k=1\}$, they let
\begin{align}\label{eq:wei}
\Upsilon(\Dcal;\lambda)=\sum_{k=1}^K \lambda_k\E_{n_k}\left[\frac{\epol(a\mid s)r}{\pi_k(a\mid s)} \right].
\end{align}
For any $\lambda\in\Delta^K$, $\Upsilon(\Dcal;\lambda)$ is unbiased under a whole weak overlap. 
\begin{assumption}[Whole Weak Overlap]\label{asm:weakwhole}
For any $s\in\Scal,\,k\in[K]$, $\epol(\cdot\mid s)\ll \pi_k(\cdot\mid s)$.
\end{assumption}
Clearly \cref{asm:weakwhole} implies \cref{asm:weak}.

Then, they consider two important special cases: the na\"ive average of the $K$ IS estimates,
$$
\hat J_{\text{IS-Avg}}:=\Upsilon(\Dcal;(n_1/n,\dots,n_k/n)), 
$$
and a precision-weighted average,
$$
\hat J_{\text{IS-PW}}:=\Upsilon(\Dcal;\lambda^*),\,
\lambda^{*}_k=\frac{n_k/\var_{\pi_{k}}[\epol(a\mid s)r/\pi_k(a\mid s)]}{\sum_{k'} n_{k'}/\var_{\pi_{k'}}[\epol(a\mid s)r/\pi_{k'}(a\mid s)] }.
$$
Notice that $\lambda^{*}=\argmin_{\lambda \in \Delta^K}\var[\Upsilon(\Dcal;\lambda)]$. Unlike $\hat J_{\text{IS}}$ and $\hat J_{\text{IS-Avg}}$, the estimator $\hat J_{\text{IS-PW}}$ is not feasible in practice since $\lambda^*$ needs to be estimated from data first (we discuss this in more detail in \cref{sec:feasible} and show that asymptotically there is no inflation in variance).

\citet{AgarwalAman2017EEUL} established two relationships about the above:
$$
\var[\hat J_{\text{IS-Avg}}]\geq \var[\hat J_{\text{IS}}],\quad
\var[\hat J_{\text{IS-Avg}}]\geq \var[\hat J_{\text{IS-PW}}].
$$
However, they noted that they cannot find a theoretical relationship between $\var[\hat J_{\text{IS}}]$ and $\var[\hat J_{\text{IS-PW}}]$. In fact, unlike the above two relationships, which of these two estimators has smaller variance \emph{depends} on the problem instance. This brings up an apparent \emph{dilemma}: which one should we use? We resolve this dilemma by showing another estimator dominates both. In fact, it dominates a much bigger class of estimators, that includes $\hat J_{\text{IS}},\Upsilon(\Dcal;\lambda),\hat J_{\text{IS-Avg}},\hat J_{\text{IS-PW}}$.

\section{Optimality} \label{sec:efficiency}

We next tackle the question of what would be the \emph{optimal} estimator. We tackle this from three perspectives. First, we study a class of estimators like $\Upsilon(\Dcal;\lambda)$ but larger, allowing for control variates, and determine the single estimator with minimal (non-asymptotic) MSE among these. Second, since not all estimators (including this optimum) are feasible in practice as they may involve unknown nuisances (just like $\hat J_{\text{IS-PW}}$ depends on the unknown $\lambda^*$), we then consider a class of feasible estimators given by plugging in these nuisances and we show that asymptotically the minimum MSE is the same and achievable. Third, we show that this minimum is in fact the efficiency lower bound, that is, the minimum asymptotic MSE among all regular estimators.
\cref{fig:relation} illustrates the relationship between these different classes of estimators.

\begin{figure}
    \centering
    \resizebox{0.7\linewidth}{!}{
\begin{tikzpicture}[x=70,y=53]
        \draw[emerald,line width=2] (0,0) circle (1.2);
        \draw[red,line width=2] (0:1.2) circle (1.6);
        \draw[blue,line width=2] (0.6,0)  ellipse (1.81 and 1.81);
        \draw (1.4,1.0) circle (0.3) node[align=center] {DR}; 
        \draw (2.0,0.5) circle (0.3) node[align=center]   {DR-Avg}; 
     \draw (2.0,-0.5) circle (0.3) node[align=center]  {IS-PW(f)}; 
        \draw (1.4,-1.0) circle (0.3) node[align=center] {DR-PW}; 
        \draw (0.44,0.6) circle (0.3) node[align=center] {IS}; 
       \draw (0.44,-0.6) circle (0.3) node[align=center] {IS-Avg}; 
         \draw (-0.7,0.0) circle (0.3) node[align=center]  {IS-PW}; 
         
     \begin{scope}
      \clip (0.6,0)  ellipse (1.81 and 1.81);
      \fill[blue,fill opacity=0.3] (0:1.2) circle (1.4);
    \end{scope}

\end{tikzpicture}
}
    \caption{Relationship between the classes of estimators considered in \cref{sec:efficiency}. The \textcolor{emerald}{green} circle represents the class $\{\Gamma(\Dcal;h,g)\}$.  The \textcolor{blue}{blue} circle is $\{\hat J_{\bi}(\hat h,\hat g)\}$. The \textcolor{red}{red} circle is regular estimators.  The \textcolor{blue}{blue} shaded region is the estimators $\hat J_{\bi}(\hat h,\hat g)$ with feasible and consistent estimators $\hat h,\hat g$ (see \cref{thm:property_feasible}). The minimal asymptotic MSE in \emph{any one} of these sets is the \emph{same} and achievable by a feasible estimator.}
    \label{fig:relation}
\end{figure}
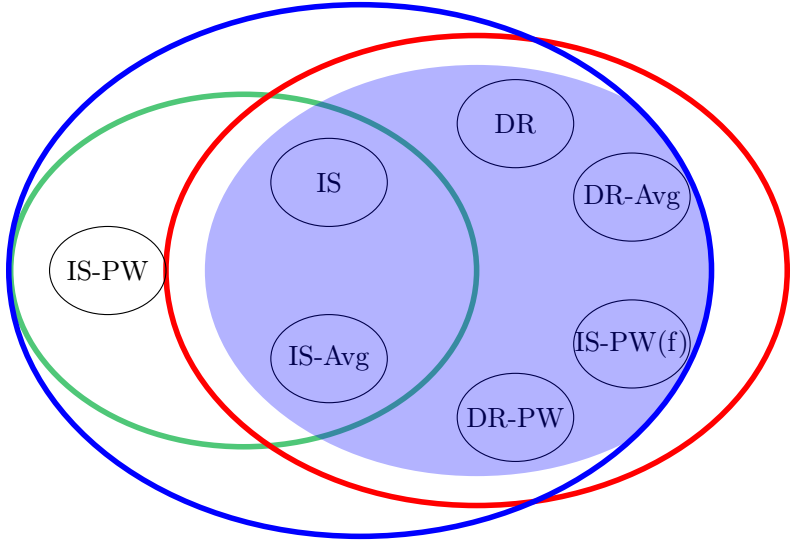

\subsection{A Class of (Possibly Infeasible) Unbiased Estimators}\label{sec: infeasible unbiased}

Consider the class of estimators given by
\begin{align*}
  \Gamma(\Dcal;h,g)=~&  \E_n[h(k,s,a)\epol(a\mid s)(r-g(s,a))+g(s,\epol)],
\end{align*}
for any choice of functions $h(k,s,a),\,g(s,a)$, where we restrict to functions $h$ that satisfy 
\begin{align}\label{eq:con}
\sum_{k=1}^Kn_k \pi_k(a\mid s)h(k,s,a)=n~\,\forall s,a:\epol(a\mid s)>0.
\end{align}
Here $h,g$ may depend on unknown aspects of the data generating distribution (\eg, $g=q$). Thus, certain choices may be infeasible in practice. Feasible analogues may be derived by estimating $h,g$ and plugging the estimates in as we will do in the next section.
We refer to the class of estimator as we range over $h,g$ satisfying \cref{eq:con} as $\{\Gamma(\Dcal;h,g)\}$, and
we refer to $h$ as ``weights'' and $g$ as ``control variates.''

This is a fairly large class in the sense that it allows both for flexible weights that depend on logger identity and for control variates. In fact, it includes the class $\Upsilon(\Dcal;\lambda)$ as a subclass (including $\hat J_{\ips},\hat J_{\wei}$) by letting $h(k,s,a)=1/\pi_k$ or $h(k,s,a)=n\lambda^{*}_k/(n_k\pi_k(a\mid s))$, and $g=0$. It also includes $\hat J_{\bal}$ by letting $h_k(k,s,a)=1/\pi_*(a\mid s)$ and $g=0$.
This class of estimators is unbiased, \ie, $\E\Gamma(\Dcal;h,g)=J$. 
But notice that the restriction on $h$ (\cref{eq:con}) implicitly requires a form of $h$-specific overlap. \Eg, for $h(k,s,a)=1/\pi_*(a\mid s)$, it corresponds to \cref{asm:weak}, and for $h(k,s,a)=n\lambda^{*}_k/(n_k\pi_k(a\mid s))$, it is implied by \cref{asm:weakwhole}.



We have the following optimality result.
\begin{theorem}\label{thm:concrete}
Suppose \cref{asm:weak} holds.
The minimum of the variances among estimators in the class $\{\Gamma(\Dcal;h,g)\}$ is $V^*/n$ where $$ V^*:=\E_{\pi_*}\bracks{\braces{\frac{\epol(a\mid s)}{\pi_{*}(a\mid s)}}^2\sigma_r^2(s,a)}+\var_{p_S}[v(s)].$$
This minimum is achieved by $\Gamma(\Dcal;1/\pi_*(s,a),q(s,a))$.
\end{theorem}

The result is remarkable in two ways. First, it gives an answer to the dilemma outlined in \cref{sec:preli}. In the end, none of the three estimators $\hat J_{\wei},\hat J_{\bal},\hat J_{\ips}$ studied by \citep{AgarwalAman2017EEUL} are optimal. Second, it states the surprising fact that logger identity information does \emph{not} contribute to the lower bound. In other words, whether we allow different weights in different strata (allow $h$ to depend on $k$), the minimum variance is unchanged since it is achieved by a stratum-\emph{independent} weight function.

This key observation can also be translated to the multiple logger settings in infinite-horizon RL \citep[\eg, as studied by][]{ChenXinyun2019IOPE}. We provide this extension in \cref{sec:extensions}.

\subsection{A Class of Feasible Unbiased Estimators}\label{sec:feasible}

When $h,g$ depend on unknowns, such as $g=q$ as in the optimal estimator in \cref{thm:concrete}, the estimator $\Gamma(\Dcal;h,g)$ is actually infeasible in practice. We therefore next study what happens when we estimate $g,h$ and plug them in. Generally, when we plug nuisance estimates in, the variance may inflate due to the additional uncertainty associated with these estimates, both in finite samples \emph{and} asymptotically: for example, when we consider a direct method estimator $\E_n[\hat q(S,\epol)]$, the asymptotic variance is much larger than $\E_n[q(S,\epol)]$. Interestingly, for the current case, this inflation does not occur asymptotically.


Specifically, we propose the feasible estimators $\hat J_{\bi}(\hat h,\hat g)$ given by the meta-algorithm in \cref{alg:feasible}, which uses a cross-fitting technique \citep{ZhengWenjing2011CTME,ChernozhukovVictor2018Dmlf}. The idea is to split the sample into a part where we estimate $g,h$ and a part where we plug them in and then averaging over different roles of the splits. If each $\hat h^{(z)}$ satisfies \cref{eq:con}, then this \emph{feasible} estimator is still unbiased since 
\begin{align*}
    \E[\Gamma(\Lcal_{z};\hat h^{(z)},\hat g^{(z)})]=  \E[\E[\Gamma(\Lcal_{z};\hat h^{(z)},\hat g^{(z)})\mid\cu_z]]=J. 
\end{align*}
If we do not use sample splitting, this unbiasedness cannot be ensured. 

In addition, in the asymptotic regime (recall that in the asymptotic regime we consider $n'=mn,\,n'_k=mn_k$ observations and $m\to\infty$) we can show that whenever $\hat h,\hat g$ are consistent, the feasible estimator $\hat J_{\bi}(\hat h,\hat g)$ is also asymptotically normal with the \emph{same} variance as the possibly infeasible $\Gamma(\Dcal;h,g)$.
\begin{algorithm}[t!]
 \caption{Feasible Cross-Fold Version of $\Gamma(\Dcal;h,g)$}
 \begin{algorithmic}[1]
 \label{alg:feasible}
 \STATE \textbf{Input}: Estimators {\blockedit $\hat h(k,s,a),\hat g(s,a)$}
\STATE Fix a positive integer $Z$. For each $k \in[K]$, take a $Z$-fold random even partition $(I_{kz})^Z_{z=1}$ of the observation indices $\{1,\dots,n_k\}$ such that the size of each fold, $\abs{I_{kz}}$, is within $1$ of $n_k/Z$
\STATE Let $\mathcal{L}_z=\{(S_{ki},A_{ki},R_{ki}):k=1,\dots, K,i\in I_{kz}\}$, $\mathcal{U}_z=\{(S_{ki},A_{ki},R_{ki}):k=1,\dots, K,i\notin I_{kz}\}$
 \FOR{$z=1,\cdots,Z$}
    \STATE Construct estimators $\hat h^{(z)}=\hat h(k,s,a;\mathcal{U}_z),\,\hat g^{(z)}=\hat g(s,a;\mathcal{U}_z)$ of $h,g$ using only $\mathcal{U}_z$ as data
    \STATE Set $\hat{J}_z=\Gamma(\Lcal_z;\hat h^{(z)},\hat g^{(z)})$
    \ENDFOR 
    \STATE \textbf{Return}:
    $\hat J_{\bi}(\hat h,\hat g) = \frac{1}{n}\sum^Z_{z=1}\abs{\Lcal_z}\hat{J}_z.$
\end{algorithmic}
\end{algorithm}

\begin{theorem}\label{thm:property_feasible}
Suppose $\|\hat h^{(z)}-h\|_2=\op(1)$, $\|\hat g^{(z)}-g\|_2=\op(1)$, $\hat h^{(z)},\hat g^{(z)},h,g$ are uniformly bounded by some constants, and $h,\hat h^{(z)}$ satisfy \cref{eq:con}. Then, $\hat J_{\bi}(\hat h,\hat g)$ is unbiased and $$\sqrt{n'}(\hat J_{\bi}(\hat h,\hat g) - J)\stackrel{d}{\rightarrow} \Ncal(0,n\var[\Gamma(\Dcal;h,g)] ).$$
\end{theorem}
{\blockedit Note the restriction on $\hat h^{(z)}$ implicitly assumes we know logging policies.}
\Cref{thm:property_feasible,thm:concrete} together immediately lead to two important corollaries:
\begin{corollary}\label{cor:property_feasible}
Under the assumptions of \cref{thm:property_feasible}, $\hat J_{\bi}(\hat h,\hat g)$ has asymptotic MSE lower bounded by $V^*$.
\end{corollary}
\Cref{cor:property_feasible} shows that among the class $\{\hat J_{\bi}(\hat h,\hat g)\}$, $V^*$ is also an MSE lower bound. This class is \emph{larger} than $\{\Gamma(\Dcal;h,g)\}$ since we can always take $\hat h=h,\hat g=g$ although it may be infeasible in practice.
\begin{corollary}\label{cor:property_feasible2}
Suppose $g=q$ and $\|\hat q^{(z)}-q\|_2=\op(1)$, \cref{asm:weak} holds, and $\hat q^{(z)},1/\pi_{*},q$ are uniformly bounded by some constants. Then, the cross-fitting doubly robust estimator $$\hat J_\DR:=\hat J_{\bi}(1/\pi_{*},\hat q)$$ achieves the asymptotic variance lower bound $V^*$. 
\end{corollary}
\Cref{cor:property_feasible2} shows that, when the logging policies are known, the minimum MSE is achievable by the cross-fitting doubly robust estimator $\hat J_\DR$.
In \cref{sec:more}, we discuss how to estimate $\hat q$, which is a necessary ingredient in constructing $\hat J_{\DR}$.

\Cref{thm:property_feasible} can also be used to establish new theoretical results about other (suboptimal) estimators.
For example, we can consider a feasible version of $\hat J_{\wei}$, which we call $\hat J_{\wei(f)}$, where we use $\hat\lambda^{*}_k=\frac{n_k/\var_{n_{k}}[\epol(a\mid s)r/\pi_k(a\mid s)]}{\sum_{k'} n_{k'}/\var_{n_{k'}}[\epol(a\mid s)r/\pi_{k'}(a\mid s)]}$. \Cref{thm:property_feasible} shows it has the \emph{same} asymptotic variance as $\hat J_{\wei}$, which was not established in \citet{AgarwalAman2017EEUL}. Additionally, we can consider the naively weighted and precision-weighted average of the doubly robust estimators in each dataset, respectively:
\begin{align*}
\hat J_{\dr\text{-Avg}}& :=\hat J_{\bi}(1/\pi_k(a\mid s),\hat q),\\
\hat J_{\dr\text{-PW}}&  :=\hat J_{\bi}(n_k\hat \lambda^{\dagger}_k/(n\pi_k(a\mid s)), \hat q),\\
 \hat \lambda^{\dagger}_k& :=\frac{n_k\var_{n_k}[\epol r/\pi_{k}\{r-\hat q(s,a)\}+\hat q(s,\epol)]}{\sum_{k'} n_{k'}\var_{n_{k'}}[\epol r/\pi_{k'}\{r-\hat q(s,a)\}+\hat q(s,\epol)]}. 
\end{align*}
These have the same asymptotic variance as $\Gamma(\Dcal;1/\pi_k(a\mid s), q),\Gamma(\Dcal;n_k \lambda^{\dagger}_k/(n\pi_k(a\mid s)),q)$, respectively, where $\lambda^{\dagger}_k$ is the same as $\hat \lambda^{\dagger}_k$ with $\var_{n_k}$ replaced with $\var_{\pi_k}$. Neither, however, is optimal and $\hat J_{\bi}(1/\pi_{*},\hat q)$ outperforms these both.

Even if the estimators $\hat h^{(z)}$ does not satisfy \cref{eq:con}, as long as the convergence point $h$ satisfies \cref{eq:con}, the final estimator is consistent, but it may not be asymptotically normal. In this case, we need additional conditions on the convergence rates to ensure $\sqrt{n'}$-consistency. This is relevant when the logging policies are not known. We explore this in \cref{sec:efficient_estimation}.

\subsection{The Class of Regular Estimators}\label{sec: regular est}

The previous sections considered the minimal MSE in a class of estimators given explicitly by a certain structure or by a meta-algorithm. We now show that the same minimum in fact reigns among the asymptotic MSE of (almost) \emph{all} estimators that are feasible in that they ``work'' for all data-generating processes (DGPs). 

Recall our data is drawn from
$$
\Dcal\sim\prod_{k=1,i=1}^{K,n_k}p_{S}(s_{ki})\pi_k(a_{ki}\mid s_{ki})p_{R\mid S,A}(r_{ki}\mid s_{ki},a_{ki}),
$$
and that in the asymptotic regime we consider observing $m$ independent copies of $\Dcal$ (for total data size $n'=mn$).
Consider first the case where $\pi_k$ are known.
Then, $p_S$ and $p_{R\mid S,A}$ are the only unknowns in the above DGP. That is, different instances of the problem are given by setting these two to different distributions. 
Thus, in the known-logger case, we consider the model (\ie, class of instances) given by all DGPs where $p_S$ and $p_{R\mid S,A}$ vary arbitrarily and $\pi_k$ are fixed.
(This is a \emph{nonparametric} model in that these distribution are unrestricted.)
Regular estimators are those that are $\sqrt{n'}$-consistent for \emph{all} DGPs and remain so under perturbations of size $1/\sqrt{n'}$ to the DGP \citep[for exact definition see][]{VaartA.W.vander1998As}. 
When $\hat h,\hat g$ satisfy the conditions of \cref{thm:property_feasible} for every instance (\ie, are feasible consistent estimators for $h,g$ for all instances), $\hat J_\bi(\hat h,\hat g)$ is a regular estimator, as a consequence of \cref{thm:property_feasible} and \citet[Lemma 8.14]{VaartA.W.vander1998As}.




The benefit of considering the class of regular estimators is that it allows us to appeal to the theory of semiparametric efficiency in order to derive the minimum asymptotic MSE in the class. We paraphrase the key result for doing this below. We provide additional detail in \cref{sec:efficiency_detail}.

\begin{theorem}{\citep[Theorem 25.20]{VaartA.W.vander1998As}}\label{thm:van}
Given a model,
the efficient influence function (EIF), $\tilde\phi(\Dcal)$, is the least-$L_2$-norm gradient of $J$ with respect to instances ranging in the model. The EIF satisfies that for any estimator $\hat J$ that is regular with respect to the model, the variance of the limiting distribution of $\sqrt{n'}(\hat J-J)$ is at least $n\var[\tilde\phi(\Dcal)]$.
\end{theorem}
The term $n\var[\tilde\phi(\Dcal)]$ is called the \emph{efficiency bound}. Estimators that achieve this bound are called \emph{efficient}. We next derive the EIF and efficiency bound for our problem. That is, for our average-reward estimand $J$ in the model given by varying $p_S,\,p_{R\mid S,A}$ arbitrarily.

\begin{theorem}\label{thm:efficient}
{\blockedit 
Define \begin{equation}\phi(s,a,r;g)=\frac{\epol(a\mid s)}{\pi_{*}(a\mid s)}(r-g(s,a))+g(s,\epol).\label{eq:double_robust}\end{equation}
Then in the model with $\pi_k$ known and fixed, the EIF is $\tilde\phi(\Dcal)=\frac{1}{n}\sum_{k=1,i=1}^{K,n_k}\phi(S_{ki},A_{ki},R_{ki};q)-J$ and the efficiency bound is $V^*$.
}
\end{theorem}
Notably, the EIF belongs to the class $\{\Gamma(\Dcal;h,g)\}$ and is exactly the optimal (infeasible) estimator in that class. Correspondingly, the efficiency bound is exactly the same $V^*$ from \cref{thm:concrete,cor:property_feasible}. This shows that, remarkably, $\hat J_\dr$ is in fact also optimal in the much broader sense of semiparametric efficiency.

Notice that in efficiency theory for OPE in the standard \emph{un}stratified case \citep{KallusUehara2019} and in other standard semiparametric theory \citep{bickel98,TsiatisAnastasiosA2006STaM}, we must consider iid sampling of observations. However, in the stratified case the data are \emph{not} iid, since $n_k$ are fixed. To be able to tackle the stratified case meaningfully we consider a dataset of size $n'\to\infty$ where the proportions of data from each logger, $\rho_k$, are always \emph{fixed}. We achieve this in a new way via the equivalent construction of observing $m$ independent copies of $\Dcal$ with $m\to\infty$.




Next, we consider the case where the logging policies $\pi_k$ are \emph{not} known. Namely, we consider the model where we allow \emph{all} of $p_S, p_{R\mid S,A}, \pi_1, \dots, \pi_K$ to vary arbitrarily. We next show that in this larger model, the EIF and efficiency bounds are again the same. 
\begin{theorem}\label{thm:b_known}
When the logging policies are not known, the EIF and the efficiency bound are the same as the ones in \cref{thm:efficient}.
\end{theorem}
Recall that \cref{thm:property_feasible} shows that the efficiency bound is asymptotically achieved with $\hat h=1/\pi_{*},\,\hat g=\hat q$ when we know each $\pi_k$. In the next section, we show that this lower bound can be achieved even if we do not know the logging policies and we use $\hat h=1/\hat\pi_{*}$ under some additional mild conditions. 



\section{Efficient and Robust Estimation with Unknown Logging Policies }\label{sec:efficient_estimation}

In the previous section, we showed the efficiency bound is the same whether we know or do not know the logging policies, but the efficient estimator proposed, $\hat J_\dr=\hat J_{\bi}(1/\pi_{*},\hat q)$, only works when they are known. A natural estimation way when we do not know behavior policies is to estimate $\pi_*$: $$ \hat J_\drpihat:=\hat J_{\bi}(1/\hat \pi_*,\hat q).$$ 
%
First, we prove efficiency of $\hat J_{\dr}$ under lax nonparametric rate conditions for the nuisance estimators. 

\begin{theorem}[Efficiency]\label{thm:efficiency}
Suppose $1/\pi_{*},q,\hat q^{(z)}, 1/\hat \pi^{(z)}_*$ are uniformly bounded by some constants and that \cref{asm:weak} holds. Assume $\forall z\in [Z]$, $\|\hat q^{(z)}-q\|_2=\op(1)$, $\|\hat \pi^{(z)}_*-\pi_*\|_2=\op(1)$, and $\|\hat q^{(z)}-q\|_2\|\hat \pi^{(z)}_*-\pi_*\|_2=\op(n'^{-1/2})$. Then, $\hat J_\drpihat$ is efficient: $\sqrt{n'}(\hat J_\drpihat-J)\stackrel{d}{\rightarrow}\Ncal(0,V^*)$. 
\end{theorem}
First, notice that \cref{cor:property_feasible2} can also be seen as corollary of \cref{thm:efficiency} by noting that if we set $\hat\pi_*=\pi_*$  then $\|\hat \pi^{(z)}_*-\pi_*\|_2=0$.
Second, notice that unlike \cref{thm:property_feasible}, we do \emph{not} restrict $\hat h=1/\hat\pi_*$ to satisfy \cref{eq:con}, as indeed satisfying it would be impossible when $\pi_k$ are unknown. At the same time, $\hat J_\drpihat$ is not unbiased (only asymptotically). Finally, notice that again $\hat J_\drpihat$, an efficient estimator, does not appear to use logger identity data. We will, however, use it in \cref{sec:more} to improve $q$-estimation.


Next, we prove double robustness of $\hat J_{\drpihat}$. This suggests when we posit parametric models for $\hat q,\hat \pi$, as long as either model is well-specified, the final estimator $\hat J_{\drpihat}$ is $\sqrt{n'}$-consistent though might not be efficient. 
This is formalized as follows noting that well-specified parametric models converge at rate ${n'}^{-1/2}$. 


\begin{theorem}[Double Robustness]\label{thm:double_robustness}
Suppose \cref{asm:weak} holds. Assume  $\forall z\in [Z]$, for some $q^{\dagger},\pi^{\dagger}_{*}$, $\|\hat q^{(z)}-q^{\dagger}\|_2=\bigO_p(n'^{-1/2})$ and $\|\hat \pi^{(z)}_*-\pi^{\dagger}_*\|_2=\bigO_p(n'^{-1/2})$, and 
$1/\pi^{\dagger}_{*},q^{\dagger},\hat q^{(z)}, 1/\hat \pi^{(z)}_*$ are uniformly bounded by some constants. Then, as long as either $q^{\dagger}=q$ or $\pi^{\dagger}_{*}=\pi_*$, $\hat J_\drpihat$ is $\sqrt{n'}$-consistent. 
\end{theorem}

\section{Stratified vs iid Sampling}\label{sec:multi}

We next discuss in more detail the differences and similarities between stratified and iid sampling. To make comparisons, consider the alternative iid DGP: $\Dcal'=\{(S_i,A_i,R_i):i=1,\dots,n\}$, where $(S_i,A_i,R_i)\sim p_{S}(s)\pi_{*}(a\mid s)p_{R\mid S,A}(r\mid s,a)$ independently for $i=1,\dots,n$. That is, we observe $n$ iid samples from the logging policy $\pi_*$. This is equivalent to randomizing the dataset sizes as $(n_1,\dots,n_K)\sim\operatorname{Multinomial}(n,\rho_1,\dots,\rho_K)$. In this iid setting, the results of \citet{KallusUehara2019} show that the efficiency bound is the \emph{same} $V^*$ as in \cref{thm:concrete,thm:property_feasible,thm:efficient} and that $\hat J_{\drpihat},\hat J_{\dr}$ also achieve this bound in the iid setting.


This is very surprising since usually an estimator has different variances in different DGPs. For example, the variance of $\hat J_{\bal}$ under the two different sampling settings are \emph{different}, \ie:
\begin{align*} 
   \var_{\Dcal}[\hat J_{\bal}]&=  \frac{1}{n}\sum_{k=1}^K \rho_k\var_{\pi_k}\bracks{\frac{\epol(a\mid s)r}{\pi_{*}(a\mid s)}}\\
   & \leq  \frac{1}{n}\var_{\pi_{*}}\left[\frac{\epol(a\mid s)r}{\pi_{*}(a\mid s)} \right]= \var_{\Dcal'}[\hat J_{\bal}]. 
\end{align*}
This inequality is easily proved by law of total variance and shows that the variance under stratified sampling is \emph{lower}. The inequality is generally strict when $\pi_k$ are distinct.
This observation generalizes.
\begin{theorem}\label{thm:stratified}
Suppose \cref{asm:weak} holds. Consider the class of estimators $\{ \E_n\bracks{\phi(s,a,r;g)}\}$, where $\phi$ is given in \cref{eq:double_robust} and $g$ is any function. Estimators in this class are unbiased. In addition, we have 
\begin{align}\label{eq:vardiff}
         \var_{\Dcal}[\E_n\bracks{\phi(s,a,r;g)} ] \leq   \var_{\Dcal'}[\E_n\bracks{\phi(s,a,r;g)} ].
\end{align}
And, equality holds for all $\epol,\pi_*$ satisfying \cref{asm:weak} if and only if $g=q$.
\end{theorem}

We have already seen the ``if'' part of the last statement.
The intuition for the ``only if'' part is that the difference in \cref{eq:vardiff}, $\var[\E[\E_n\bracks{\phi(s,a,r;g)}\mid \{n_k\}_{k=1}^K]]$, is zero exactly when $\E_{\pi_k}[\phi(s,a,r;g)]=J\,\forall k\in [K]$, which can only happen for any $\pi_*$ if $g=q$ so we get unbiasedness due to double robustness even for a ``wrong'' importance weight. This conveys two things: stratification is still beneficial in reducing variance in finite samples since we never know the true $q$ exactly, while at the same time the efficiency bound is the same in the two settings so this reduction washes out asymptotically when we use an efficient estimator, but \emph{only} if we use an efficient estimator. 


\section{Stratified More Robust Doubly Robust Estimation}\label{sec:more}

We have so far considered a meta-algorithm for efficient estimation given a $q$-estimator, which can be constructed by applying any type of off-the-shelf nonparametric or machine learning regression method to the whole dataset $\Dcal$. 
However, if $\hat q$ is misspecified and inconsistent, the theoretical guarantees such as efficiency fail to hold. This a serious concern in practice as we always risk some level of model misspecification. We therefore next consider a more tailored loss function for $q$-estimation that can still provide intrinsic efficiency guarantees regardless of specification.


Specifically, following \citet{rubin08,CaoWeihua2009Iear,Chow2018}, we consider choosing the control variate $g$ in a hypothesis class $\Qcal$ to minimize the variance of $\Gamma(\Dcal;1/\pi_*,g)=\E_n[\phi(s,a,r;g)]$.
Specifically, we are interested in: 
\begin{align*}
 \tilde{q}:=\argmin_{g\in \Qcal}V(g),\,V&(g)=n\var[\E_n[\phi(s,a,r;g)]]\\
& =\sum_{k=1}^K\rho_k\var_{\pi_k}[\phi(s,a,r;g)].
\end{align*}
Of course, per \cref{thm:concrete}, if $q\in\Qcal$ then $\tilde{q}=q$, but the concern is that $q\notin\Qcal$. In this case, $\tilde q$ will ensure best-in-class variance and will generally perform better than the best-in-class regression function $\bar q=\argmin_{g\in \Qcal}\E_{\pi_*}[(r-g(s,a))^2]$, which empirical risk minimization would estimate.

In practice, we need to estimate $\var_{\pi_k}[\phi(s,a,r;g)]$. 
A feasible estimator is 
\begin{align*}
 &  \check q:=\argmin_{g\in \Qcal}
\sum_{k=1}^K\rho_k\var_{n_k}[\phi(s,a,r;g)]. 
\end{align*}
Then, we define the \emph{Stratified More Robust Doubly Robust} estimator as $\hat J_{\mo}:=\hat J_{\bi}(1/\pi_*,\check q)$. 

\begin{theorem}\label{thm:more_doubly}
Suppose $1/\pi_{*},\sup_{g\in\Qcal}\abs{g(s,a)}$ are uniformly bounded by some constants and \cref{asm:weak} holds. Assume a condition for the uniform covering number: $\sup_{U}\log N(\epsilon,\Qcal,L_2(U))\lesssim (1/\epsilon)$, where $N(\cdot)$ is a covering number and the supremum is taken over all probability measures. Then, $\sqrt{n'}(\hat J_{\mo}-J)\stackrel{d}{\rightarrow}\mathcal{N}\left(0,\min_{g\in \Qcal}V(g)\right)$. 
\end{theorem}

Notice that if we had ignored the stratification and used the standard MRDR estimator \citep{CaoWeihua2009Iear}, we would end up minimizing the \emph{wrong} objective:
\begin{align*}
\check q_{\mmo}:=\argmin_{g\in \Qcal}\var_n[\phi(s,a,r;g)],
\end{align*}
which targets the variance under iid sampling. In particular, we will \emph{not} obtain the best-in-class variance. This is again a consequence of \cref{thm:stratified}: when the control variates is not \emph{exactly} $q$, the variances under stratified and iid setting are \emph{different}.

\section{Experimental Results} \label{sec:experiment}
We next empirically compare our methods with the existing estimators for OPE with multiple loggers. 

\paragraph{Setup.}
Following previous work on OPE~\citep{Chow2018,wang2017optimal,Kallus2019IntrinsicallyES} we evaluate our estimators using multiclass classification datasets from the UCI repository. 
Here we consider the optdigits and pendigits datasets (see \cref{tab:dataset} in Appendix~\ref{app:experiment}.).
We transform each classification dataset into a contextual bandit dataset by treating the labels as actions and recording reward of 1 if the correct label is chosen by a classifier, and 0 otherwise.
This lets us evaluate and compare several different estimators with ground-truth policy value of an evaluation policy. 
\begin{figure*}[t!]\begin{minipage}[b]{.5\textwidth}
    \centering
    \includegraphics[clip,width=\textwidth]{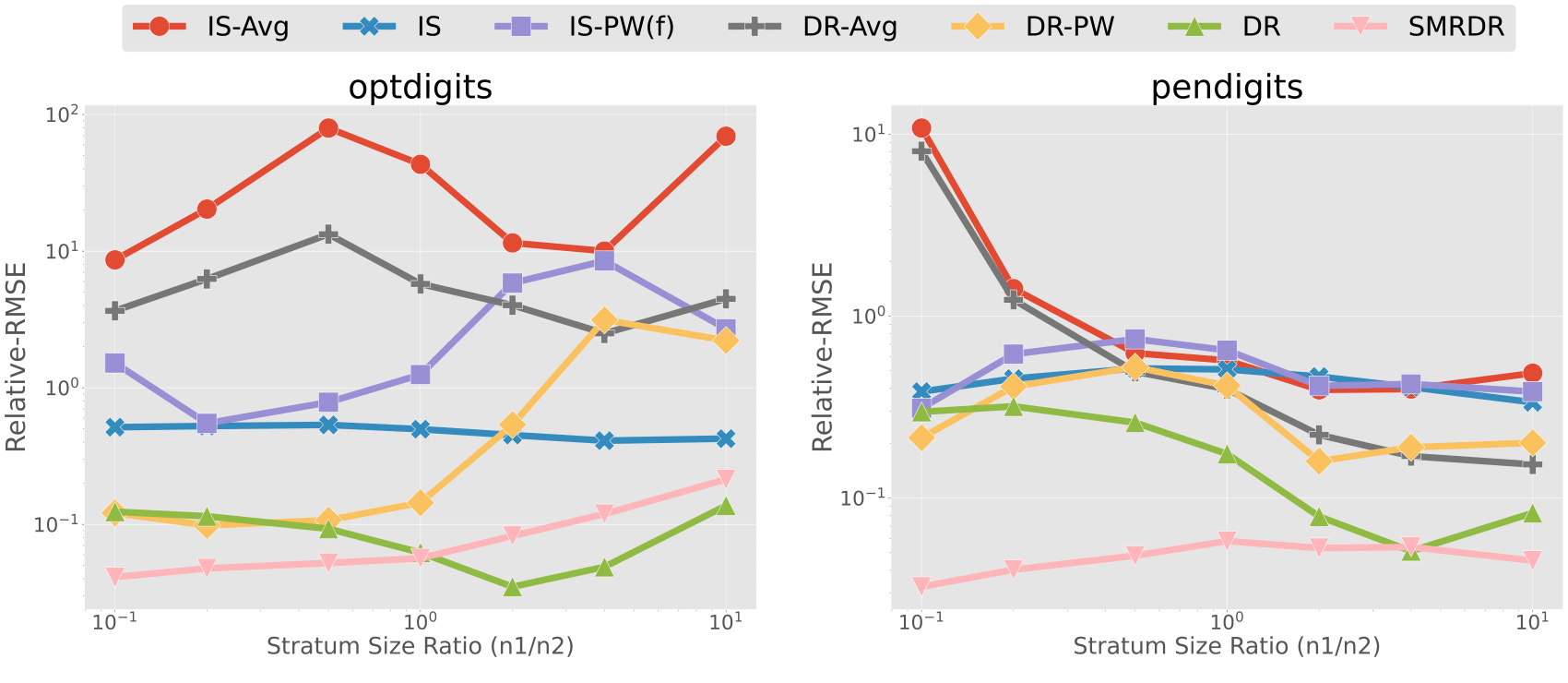}
    \caption{Comparing proposed estimators to some variants of IS type estimators.}
    \label{fig:dr_vs_is}
    \end{minipage}\begin{minipage}[b]{.5\textwidth}
    \includegraphics[clip,width=\textwidth]{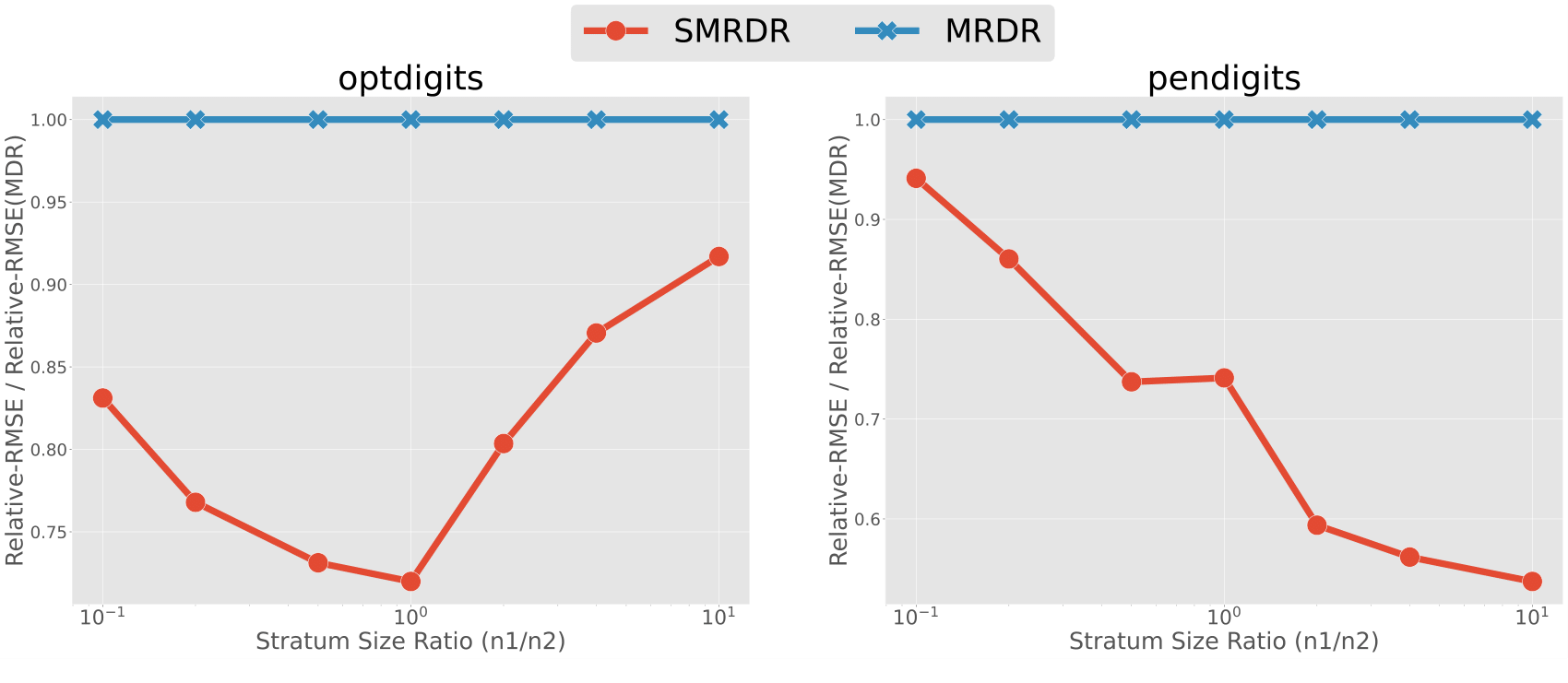}
    \caption{Comparing SMRDR (leveraging the stratification) and MRDR (ignoring the stratification).}
    \label{fig:smdr_vs_mdr}
\end{minipage}\end{figure*}

We split the original data into training (30\%) and evaluation (70\%) sets.
We first obtain a deterministic policy $\pi_{\dett}$ by training a logistic regression model on the training set.
Then, following Table~\ref{tab:policy_params}, we construct evaluation and logging policies as mixtures of one of the deterministic policy and the uniform random policy $\pi_u$.
We vary $\rho_1/(1-\rho_1)=n_1/n_2$ in $\{0.1, 0.25, 0.5, 1, 2, 4, 10\}$.  Since $\pi_1$ is closer to $\epol$ than $\pi_2$, larger $\rho_1/\rho_2$ corresponds to an easier problem.
We then split the evaluation dataset into two according to proportions $\rho_1,\rho_2$ and in each dataset we use the corresponding policy to make decisions and generate reward observations (the true label is then omitted). Using the resulting dataset we consider various estimators $\hat J$ for $J$.
{We describe additional details of the experimental setup in Appendix~\ref{app:experiment}.}

We repeat the process $M=200$ times with different random seeds and report the \textit{relative root MSE}:
\begin{align*}
    \text{Relative-RMSE} \ (\hat{J}) = \frac{1}{J\sqrt M}\sqrt{ \sum_{m=1}^M \left({J - \hat{J}_m} \right)^2 }
\end{align*}
where $\hat{J}_m$ is an estimated policy value with $m$-th data.

\begin{table}[t!]
  \centering
  \caption{The evaluation and logging policies used in the experiments.}
  \scalebox{0.95}{\begin{tabular}{ll} 
  \toprule
     evaluation policy ($\pi_e$)& $1.00 \pi_{\dett} + 0.00\pi_u$  \\
     logging policy 1 ($\pi_{1}$) & $0.95 \pi_{\dett} + 0.05\pi_u$\\
     logging policy 2 ($\pi_{2}$) & $0.05 \pi_{\dett} + 0.95\pi_u$  \\ 
  \bottomrule
  \end{tabular}}
  \label{tab:policy_params}
\end{table}

\paragraph{Estimators considered.}
We consider the following estimators:
\begin{itemize}
    [leftmargin=*,,labelindent=0in,itemsep=1ex,partopsep=0ex,parsep=0ex,topsep=0ex]
    \item Our proposed estimators, {\blockedit $J_\drpihat,\hat J_\mo$.}
    \item Standard estimators in the iid setting, $\hat J_\bal,\hat J_\mmo$.
    \item (Feasible versions of) the two estimators proposed by~\citep{AgarwalAman2017EEUL}, $\hat J_\text{IS-Avg},\hat J_{\text{IS-PW}}$.
    \item The natural doubly robust extension of these as discussed in Section~\ref{sec:feasible}, $\hat J_{\dr\text{-Avg}},\hat J_{\dr\text{-PW}}$.
\end{itemize}

{\blockedit We suppose we do not know logging policies.}
For all estimators, we estimate the logging policies using logistic regression on the evaluation set with 2-fold {cross-fitting} as in Algorithm~\ref{alg:feasible}.
{\blockedit Most of the estimators above are introduced with known logging densities in the previous sections. Here, we just replace each $\pi_k$ with their estimates.}
For DR, DR-Avg, and DR-PW, we construct $q$-estimates using logistic regression again using 2-fold {cross-fitting} as in Algorithm~\ref{alg:feasible}.
For SMRDR and MRDR, we optimize their respective estimated variance objectives over the class of logistic regression $\Qcal$.
We use \textit{tensorflow} and the same hyperparameter setting for DR, DR-Avg, DR-PW, SMRDR, and MRDR to ensure a fair comparison.

\paragraph{Results.}
The resulting Relative-RMSEs on optdigits and pendigits datasets with varying values of $n_1/n_2$ are given in \cref{fig:dr_vs_is,fig:smdr_vs_mdr}.
Several findings emerge from the results.
First, we see the dilemma pointed out by \citet{AgarwalAman2017EEUL}:
Specifically, the ordering of the variances of IS-Avg and IS-PW depend on the instance. More generally, there is no clear ordering between IS, IS-Avg, IS-PW, DR-Avg, and DR-PW. 
For example, on the optdigits data, DR-PW performs best among baselines with small values of $n_1/n_2$, while IS performs better with large values of $n_1/n_2$. This behavior is predicted by our analysis showing none of these estimators are optimal.

Second, our proposed estimators successfully resolve the dilemma and are superior to the above suboptimal estimators.
Moreover, we see SMRDR generally performs better than DR, especially when overlap is weak ($n_1/n_2$ is small), which exacerbates issues of misspecification. It does appear that DR outperforms SMRDR in the specific example of optdigits when overlap is strong ($n_1/n_2$ is large), which might be attributed to bad optimization of the non-convex objective compared to a reasonably good-enough plug-in $q$-estimate.

Finally, we directly compare the performances of SMRDR and MRDR in Figure~\ref{fig:smdr_vs_mdr}.
We observe that SMRDR significantly outperforms MRDR in the stratified setting, leading to up to 45\% reduction in error.
This strongly highlights that even though the asymptotic efficiency bounds are the same in the stratified and iid settings, leveraging the stratification structure can still offer significant gains in the multiple logger setting.


\section{Conclusions and Future Directions}

We studied OPE in the multiple logger setting, framing it as a form of stratified sampling. We then studied optimality in several classes of estimators and showed that, at least asymptotically, the minimum MSE is the same among all of them. We proposed feasible estimators that can achieve this minimum, whether logging policies are known or not. This gives a concrete and positive resolution to the multiple logger dilemma posed in \citet{AgarwalAman2017EEUL}. We further discuss how to take stratification into account when choosing best-in-class control variates.

There are a number of avenues for future work.
One is to consider optimality in the case of adaptive data collection from multiple loggers, where each logger may depend on historical data so far \citep{luedtke2016,HadadVitor2019CIfP,ZhangKellyW2020IfBB,kato2020adaptive}. Another is to study off-policy optimization in the stratified setting, whether by policy search \citep{ZhangBaqun2013Reoo,kallus2018balanced,kallus2017recursive,HeLi2019OLfM} or by off-policy gradient ascent \citep{EEOPG2020}.

\clearpage 
\bibliographystyle{chicago}
\bibliography{rc}

\newpage 
\appendix 
\onecolumn 

\section{Notation}

We first summarize the notation we use in \cref{tab:pre}.   

\begin{table}[h!]\label{tab:pre}
    \centering 
      \caption{Notation} \vspace{0.3cm}
    \begin{tabular}{l|l}
    $n, n_k,\,(1\leq k\leq K),K$   &  Whole sample size, sample  size, stratification size \\
     $n', n'_k,\,(1\leq k\leq K),K$   &  Sample size considering asymptotics\\
    $\rho_k$ & $n_k/n$\\
        $J$ & Policy value $\E_{\epol}[r]$ \\
    $p_{S}(s),p_{R\mid S,A}(r\mid s,a)$ & State, reward distributions \\
    $s,a,r$ & State, action, reward \\ 
    $[K]$ & Partition $[1,\cdots,K]$ \\ 
    $[Z]$ & Partition for cross-fitting \\ 
    $\Dcal=\{\Dcal_1,\cdots,\Dcal_K\},\Dcal_k=\{S_{kj},A_{kj},R_{kj}\}_{j=1}^{K,n_k}$ &  Observed whole data, data in the size $k$ \\
    $\pi_k,\pi_*,\epol$ & $k$-th behavior policy, mixture of k policies, evaluation policy \\ 
    $\E_{\pi}[f(k,s,a,r)],\var_{\pi}[f(k,s,a,r)]$ & Expectation and variance regarding $\pi$  \\
    $\hat J_{\ips},\hat J_{\bal},\hat J_{\wei}$  & Na\"ive IS, IS, precision-weighted IS estimator \\ 
    $\hat J_{\dr},\hat J_{\mo}$  & Doubly robust, Stratified more robust doubly robust estimator \\
    $\{\Upsilon(\Dcal;\lambda)\},\{\Gamma(\Dcal;h,g)\}$ &  Set of some estimators \\ 
    $\hat J_{\bi} $ & Feasible cross-fold version estimators \\
    $\E_n[f(k,s,a,r)],\E_{n_k}[f(k,s,a,r)]$ & Empirical approximation \\
    $q(s,a),v(s)$ & $\E[r|s,a],q(s,\epol)$\\
    $\Ncal(0,B)$ & Normal distribution with mean $0$ and variance $B$\\
    $\|f\|_2$ & $\{\E_{\pi_{*}}[f^2(k,s,a,r)]\}^{1/2}$ \\
    $o$ & $\{s_{jk},a_{jk},r_{jk}\}$ \\
    $\tilde \phi(o)$ & EIF: $ \frac{1}{n}\sum_{i=1}^n\braces{\frac{\epol(a_i\mid s_i)}{\pi_{*}(a_i\mid s_i)}\{r_i-q(s_i,a_i)\}+q(s_i,\epol)-J}. $\\
    $\phi(s,a,r;g)$ & $\epol/\pi\{r-g(s,a)\}+g(s,\epol)$ \\  
    $\Qcal$ & Function class for $g$ in $\{\phi(s,a,r;g)\}$ in SMRDR \\
    $\Lcal_z,\cu_z$ & Set induced by sample splitting \\
    $\sigma^2_r(s,a)$ & Variance: $\var[r \mid s,a]$ \\
    \end{tabular}
\end{table}

\newpage
\allowdisplaybreaks

\section{Off-policy Evaluation in Reinforcement Learning with Multiple Loggers }\label{sec:extensions}

We discuss an efficiency bound and a method to achieve the efficiency bound when the data is generated by an MDP and multiple loggers.  

Consider that we have a data $\Dcal=\{D_1,\cdots,D_K\}$:
\begin{align*}
    \Dcal_k=\{S_{kj},A_{kj},R_{kj},S'_{kj}\}_{j=1}^{n_k}\overset{\text{i.i.d}}\sim p_{k}(s)\pi_k(a\mid s)p_{R\mid S,A}(r\mid s,a)p_{S'|S,A}(s'\mid s,a). 
\end{align*}
When $K=1$, this is a standard DGP assumption in an RL setting. Here, there are $K$-multiple loggers. State  distribution $p_k(s)$ for each logger can be different as well as each behavior policy. We sometimes reindex the whole data as $\Dcal=\{S_i,A_i,R_i,S'_i\}_{i=1}^n$. In this section, given some function $f(s,a,r,s')$, we define
\begin{align*}
    \E_{p_{k}(s)\times \pi_k}[f(s,a,r,s')]:= \int f(s,a,r,s')p_{k}(s)\pi_k(a\mid s) p_{R\mid S,A}(r\mid s,a)p_{S'|S,A}(s'\mid s,a)\mathrm{d}(s,a,r,s'). 
\end{align*}
Our target is the policy value $J(\gamma)$ defined by the same MDP and an evaluation policy $\epol$ with a discount factor $\gamma$ as follows: 
\begin{align*}
    J(\gamma)=(1-\gamma)\lim_{T\to \infty}\E[\sum_{t=1}^T\gamma^{t-1}r_t\mid s_1\sim p_{e}(s),a_1\sim \epol(s_1),a_2\sim \epol(s_2),\cdots], 
\end{align*}
where $p_e(s)$ is an initial state distribution. Here, we have an important observation
\begin{align*}
    J(\gamma)=\E_{p^{(\infty)}_{e,\gamma}\times \epol }[r], 
\end{align*}
where $p^{(\infty)}_{e,\gamma}(s)$ is an average visitation distribution with a discount factor $\gamma$ and an initial distribution $p_e(s)$. Based on \citet{Liu2018} when $K=1$, this is estimated by 
\begin{align*}
    \frac{1}{n}\sum_{i=1}^n \hat w(S_i)\frac{\epol(A_i \mid S_i)}{\pi_1(A_i \mid S_i)}R_i
\end{align*}
where $\hat w(s)$ is some estimator for $w(s):=p^{(\infty)}_{e,\gamma}/p_{1}(s)$. In this $K=1$ setting, \citet{KallusNathan2019EBtC} derived the efficiency bound and a way to achieve the efficiency bound. 

Here, we give the efficiency bound with a multiple logger case. This is 
\begin{align*}
    \E_{\pi_{*}(s,a)}\left[\left\{\frac{p_{\epol,\gamma}(s)\epol(a\mid s)}{\pi_{*}(s,a)}\right\}^2 \var[r+\gamma q(s',\epol)\mid s,a]\right], 
\end{align*}
where $\pi_{*}(s,a)=\sum_{k=1}^K (n_k/n)\pi_k(a\mid s)p_{k}(s)$. When $K=1$, this result is reduced to \citet{KallusNathan2019EBtC}. Though we do not give a formal derivation, this is derived in the same spirit of \cref{thm:efficient}. 

Next, we give an efficient estimator. Before that, we define $w(s,a):=\left\{\frac{p^{(\infty)}_{\epol,\gamma}(s)\epol(a\mid s)}{\pi_{*}(s,a)}\right\}$, $q(s,a):=\E_{\epol}[\sum_{t=1}^{\infty} \gamma^{t-1}r_t\mid s_1=s,a_1=a]$. The efficient estimator is
\begin{align*}
      \frac{1}{n}\sum_{i=1}^n \hat w(S_i,A_i)\{R_i+\gamma \hat q(S'_{i},\epol) -\hat q(S_i,A_i)\}+\E_{p_e(s)}[\hat q(s,\epol)]  
\end{align*}
given some estimators $\hat w(s,a),\hat q(s,a)$. Q-functions are estimated by any off-the-shelf methods such as fitted Q-iteration \citep{antos2008learning}. We can estimate the ratio $w(s,a)$ using some methods agnostic to $p_{k}(s)$ and $\pi_k(a\mid s)$ following \citet{UeharaMasatoshi2019MWaQ}. More specifically, for some test function $f(s,a)$, this is estimated by solving 
\begin{align*}
    0=\frac{1}{n}\sum_{i=1}^n \braces{\gamma w(S_i,A_i;\beta)f(S'_i,\epol)-w(S_i,A_i;\beta)f(S_i,A_i)}+(1-\gamma)\E_{p_e(s)}[f(s,\epol)],
\end{align*}
w.r.t. $\beta$, where $w(s,a;\beta)$ is some model for $w(s,a)$. Here, $\pi_*$ is not included in the estimating equation. Note that this is different form the ones in \citet{Liu2018,KallusNathan2019EBtC}, which are not agnostic to $\pi_k(a\mid s)$. 

Finally, note that our result is more sophisticated comparing to \citet{ChenXinyun2019IOPE} in the sense that (1) they consider a special case when $p_{k}(s)$ is a stationary distribution; however, our result is applied to any $p_{k}(s)$, (2) their estimator does not use a control variate; however, our estimator and optimality result take the control variate term $q(s,a)$ into account.

\section{Efficiency Bound} \label{sec:efficiency_detail}

A central question is what is the smallest-possible error we can hope to achieve in estimating $J$. In parametric models, the \Cramer-Rao lower bound gives the lower bound of the variance among unbiased estimators. We have a stronger result that the \Cramer-Rao lower bound lower bounds the asymptotic MSE of all regular estimators \citep[Chapter 7]{VaartA.W.vander1998As}. Besides, this \Cramer-Rao lower bound is extended from parametric models to non or semiparametric models, which is called an efficiency bound \citep[Chaptre 25]{VaartA.W.vander1998As}. Again, this efficiency bound lower bounds the asymptotic MSE of all regular estimators. Standard semiparametric theory is established under the i.i.d sampling \citep{bickel98,TsiatisAnastasiosA2006STaM}. Since our data mechanism is not i.i.d (not identical though independent), it looks we cannot apply this theory. 

Here, the trick to apply this theory to our setting is regarding a set of $n$ samples as one observation. In other words, we consider that we have $m$ copies of this single observation consisting of $n$ samples, where $n':=nm\to\infty$ with fixed $n$ as $m\to\infty$. We consider a nonparametric model $\cm$:
\begin{align*}
    p(o)=\prod_{k=1}^{K}\prod_{j=1}^{n_k}p_{S}(s_{kj})\pi_k(a_{kj}\mid s_{kj})p_{R\mid S,A}(r_{kj}\mid s_{kj},a_{kj}),
\end{align*}
where each density is free except for the weak overlap constraint \footnote{Without this overlap, the estimand $J$ is not identifiable.}. 
We also consider another nonparametric model $\cm_b$:
\begin{align*}
    p(o)=\prod_{k=1}^{K}\prod_{j=1}^{n_k}p_{S}(s_{kj})\pi_k(a_{kj}\mid s_{kj})p_{R\mid S,A}(r_{kj}\mid s_{kj},a_{kj}),
\end{align*}
where $\pi_k$ is fixed at the true value and other densities (state and reward densitiese) are free except for the weak overlap constraint. Then, the efficiency bound of each model lower bounds the limit of the MSE for any regular estimator $\hat J$ w.r.t each model. 

To check this, we informally state this key property of the efficient influence function (EIF) in our setting.  
\begin{theorem}{Theorem~\ref{thm:van}}
The EIF $\tilde \phi(o)$ is the gradient of $J$ w.r.t the model $\cm$, which has the smallest $L_2$-norm and it satisfies that for any regular estimator $\hat J$ of J w.r.t the model $\cm$, $\mathrm{AMSE}[\hat J]\geq \var[\phi(\Dcal)]$, where $\mathrm{AMSE}[\hat J]$ is the second moment of the limiting distribution of $\sqrt{n'}(\hat J-J)$.  
\end{theorem}
Note that a regular estimator is any whose limiting distribution is insensitive to small changes of order $\bigO(1/\sqrt{m})$ to the DGP in the model \citep[Chapter 7]{VaartA.W.vander1998As}. This is a super broad class of estimators excluding pathological estimators such as Hodges' estimator. The term $\var[\phi]$ is called the efficiency bound.
For the current problem, the EIF and the efficiency bound are derived as follows. 
\begin{theorem}
Under the model $\cm$, the EIF $\tilde \phi(o)$ is 
\begin{align*}
    \frac{1}{n}\sum_{i=1}^n\braces{\frac{\epol(a_i\mid s_i)}{\pi_{*}(a_i\mid s_i)}\{r_i-q(s_i,a_i)\}+v(s_i)-J}. 
\end{align*}
The efficiency bound $V(\cm)$ scaled by $n$, i.e., $n\var[\tilde \phi(o)]$, is 
\begin{align*}
\E_{\pi_*}\bracks{\braces{\frac{\epol(a\mid s)}{\pi_{*}(a\mid s)}}^2\var[r\mid s,a]}+\var_{p_{S}}[v(s)]. 
\end{align*}
The EIF and efficiency bound are the same for the model $\cm_b$. 
\end{theorem}

We will give the formal proof in \cref{sec:proof}. Before that, we show that this is exactly the \Cramer-Rao lower bound in a finite, action, reward space setting. 

\begin{theorem}
Assume $\Scal,\Acal,\Rcal$ is a finite space. Then, the \Cramer-Rao lower bound of $J$ is $V(\cm_b)$. 
\end{theorem}

\begin{proof}
We define the \Cramer-Rao lower bound of the target functional. Assume some parametric model  $$\{p_{S}(s;\theta_{0}),\pi_1(a\mid s;\theta_1),\cdots, \pi_K(a\mid s;\theta_K),p_{R|A,S}(r|a,s;\theta_{K+1})\},$$
where each parameter corresponds to each state, action and reward. For example, assume $\Scal=\{\frs_1,\cdots,\frs_b\}$:
\begin{align*}
    p_S(s;\theta_0)=\braces{\sum_{i=1}^{b-1} I(\frs_i=s)\theta_{0i}}- I(\frs_b=s)\theta_{0b}. 
\end{align*}
The $i$-th element of the score of this $p_S(s)$ ($1\leq i\leq b-1$) is 
\begin{align*}
    \log_{\theta_{0i}} p_S(s)=I(\frs_i=s)/\theta_{0i}- I(\frs_b=s)/\theta_{0b}. 
\end{align*}
Let us define a score function for a parametric submodel:
\begin{align*}
    g_{S} &= \nabla_{\theta_0}\log p_{S}(s;\theta_{0}),\,g_{k} =  \nabla_{\theta_k}\log \pi_k(a\mid s;\theta_{k}),  g_{R|A,S} =\nabla_{\theta_{K+1}}\log p_{R|A,S}(r|a,s;\theta_{K+1}), \\
    g_{S,A,R} &= \{g^{\top}_{S},g^{\top}_{1},\cdots, g^{\top}_{K}, g^{\top}_{R|A,S} \}^{\top},\,  g_{S,A} = \{g^{\top}_{1},\cdots, g^{\top}_{K} \}^{\top},\,\theta=\{\theta^{\top}_{0},\cdots,\theta^{\top}_{K+1}\}^{\top}. 
\end{align*}
The \Cramer-Rao lower bound is defined as 
\begin{align*}
    \nabla_{\theta^{\top}}\E_{\epol}[r]I(\theta)^{-1}\nabla_{\theta}\E_{\epol}[r].
\end{align*}
 The term $I(\theta)$ is 
\begin{align*}
    I(\theta)&=\sum_{k=1}^{K}\sum_{j=1}^{n_k}\begin{pmatrix}
     \E_{\pi_k}[\nabla_{\theta^{\top}_0}g_S] & \bm{0} & 0 \\
     \bm{0} &  \E_{\pi_k}[\nabla_{\theta^{\top}_{S,A}}g_{S,A} ]   & \bm{0}\\
      0 & \bm{0} & \E_{\pi_k}[\nabla_{\theta^{\top}_{K+1}}g_{R|S,A}]\\
    \end{pmatrix}
    \\
    &=n\begin{pmatrix}
     \E_{\pi_{*}}[\otimes g_S] & \bm{0} & 0 \\
     \bm{0} &   \frac{1}{n}\sum_{k=1}^{K}\sum_{j=1}^{n_k}\E_{\pi_k}[\nabla_{\theta^{\top}_{S,A}}g_{S,A} ]  & \bm{0}\\
      0 & \bm{0} & \E_{\pi_{*}}[\otimes g_{R|S,A}]\\
    \end{pmatrix}
\end{align*}
In addition, 
\begin{align*}
    \nabla_{\theta}\E_{\epol}[r]=(\E_{\epol}[rg^{\top}_S],0,\cdots,0,\E_{\epol}[rg^{\top}_{R|S,A}] )^{\top}. 
\end{align*}
From matrix CS-inequality, this is transformed as 
\begin{align*}
&\nabla_{\theta^{\top}}\E_{\epol}[r]I(\theta)^{-1}\nabla_{\theta}\E_{\epol}[r]\\
&=\frac{1}{n} \E_{\pi_e}[rg^{\top}_{R|A,S}]\E_{\pi_{*}}[g_{R|A,S}g_{R|A,S}^{\top}]^{-1}\E_{\pi_e}[g_{R|A,S}r]+\frac{1}{n}\E_{\pi_e}[rg^{\top}_{S}]\E_{\pi_{*}}[g_{S}g_{S}^{\top}]^{-1}\E_{\pi_e}[g_{S}r]\\
 &= \frac{1}{n}\E_{\pi_{*}}\left[\frac{\epol}{\pi_{*}}\{r-q(s,a)\}g^{\top}_{R|A,S}\right]\E_{\pi_{*}}[g_{R|A,S}g_{R|A,S}^{\top}]^{-1} \E_{\pi_{*}}\left[\frac{\epol}{\pi_{*}}\{r-q(s,a)\}g_{R|A,S}\right]\\
 &+\frac{1}{n}\E_{\pi_{*}}\left[(v(s)-J)g^{\top}_{S}\right]\E_{\pi_{*}}[g_{S}g_{S}^{\top}]^{-1}\E_{\pi_*}[g_{S}(v(s)-J)]\\
 &= \frac{1}{n}\E_{\pi_{*}}\left[\left\{\frac{\epol}{\pi_{*}}(r-q(s,a))\right\}^2\right]+ \frac{1}{n}\E_{\pi_{*}}[(v(s)-J)^2]. 
\end{align*}
Here, we use the assumption that state, action and reward spaces are finite to state the last equality. For example, any function $g(s)$ s.t. $\E[g(s)]=0$ is represented as a linear combination of  $\{\log_{\theta_{0i}} p_S(s)\}_{i=1}^{b-1}$. 
\end{proof}

\section{Proof}\label{sec:proof}

\begin{proof}[Proof of \cref{thm:concrete}]

We define $\mathbf{S}:=\{S_{kj}\},\mathbf{A}:=\{A_{kj}\}$. Then, by law of total variance, the variance of $\Gamma(\Dcal;h,g)$ is decomposed as follows: 
\begin{align}
 &\var\bracks{\frac{1}{n}\sum_{k=1}^{K}\sum_{j=1}^{n_k}h(k,S_{kj},A_{kj})\epol(A_{kj} \mid S_{kj})\{R_{kj}-g(S_{kj},A_{kj})\}+g(S_{kj},\epol) \}} \nonumber \\
 &=\E\bracks{\var\bracks{\frac{1}{n}\sum_{k=1}^{K}\sum_{j=1}^{n_k}h(k,S_{kj},A_{kj})\epol(A_{kj} \mid S_{kj})\{R_{kj}-g(S_{kj},A_{kj})\}|\mathbf{S,A}}} \nonumber \\
 &+\E\bracks{\var\bracks{\E\bracks{\frac{1}{n}\sum_{k=1}^{K}\sum_{j=1}^{n_k}h(k,S_{kj},A_{kj})\epol(A_{kj} \mid S_{kj})\{R_{kj}-g(S_{kj},A_{kj})\}|\mathbf{S,A}}|\mathbf{S}}} \nonumber \\
  &+\var\bracks{\E\bracks{\frac{1}{n}\sum_{k=1}^{K}\sum_{j=1}^{n_k}h(k,S_{kj},A_{kj})\epol(A_{kj} \mid S_{kj})\{R_{kj}-g(S_{kj},A_{kj})\}+g(S_{kj},\epol)|\mathbf{S}}} \label{eq:var3}   \\ 
 &=\E\bracks{\var\bracks{\frac{1}{n}\sum_{k=1}^{K}\sum_{j=1}^{n_k}h(k,S_{kj},A_{kj})\epol(A_{kj} \mid S_{kj})R_{kj}|\mathbf{S,A}}} \label{eq:var1} \\
 &+\E\bracks{\var\bracks{\frac{1}{n}\sum_{k=1}^{K}\sum_{j=1}^{n_k}h(k,S_{kj},A_{kj})\epol(A_{kj} \mid S_{kj})\{q(S_{kj},A_{kj})-g(S_{kj},A_{kj})\}|\mathbf{S,A}|\mathbf{S}}}+ \label{eq:var2}  \\ 
 &+\var_{p_{S}(s)}\bracks{\frac{1}{n}q(s,\epol)}.  \label{eq:var4}  
\end{align}
The term \eqref{eq:var3} is converted into the term \eqref{eq:var4} because of the constraint \eqref{eq:con}. The term \eqref{eq:var2} takes $0$ when $g(s,a;p)=q(s,a)$. Thus, we only focus on the term \eqref{eq:var1}. The term \eqref{eq:var1} is further expanded as 
\begin{align*}
    &\frac{1}{n}\sum_{k=1}^K\rho_k\E_{\pi_k}[h^2(k,s,a)\{\epol(a\mid s)\}^2\var[r \mid s,a]]=\frac{1}{n} \E_{\pi_*}\bracks{\sum_{k=1}^K\frac{\rho_k\pi_k(a\mid s)h^2(k,s,a)}{\pi_*}\{\epol(a\mid s)\}^2\var[r \mid s,a] }\\
    &=\frac{1}{n}\E_{\pi_*}\bracks{\sum_{k=1}^K\frac{\rho^2_k\pi^2_k(a\mid s)h^2(k,s,a)}{\rho_k \pi_k \pi_*}\{\epol(a\mid s)\}^2\var[r \mid s,a] }\\
    &\geq \frac{1}{n}\E_{\pi_*}\bracks{\frac{\{\sum_{k=1}^K \rho_k\pi_k(a\mid s)h(k,s,a)\}^2}{\pi^2_*(a\mid s)}\{\epol(a\mid s)\}^2\var[r \mid s,a] }\\
    &=\frac{1}{n}\E_{\pi_*}\bracks{ \frac{1}{\pi^2_*(a\mid s)}\{\epol(a\mid s)\}^2 \var[r \mid s,a]}. 
\end{align*}
Here, we use CS-inequality in the second line. From the second line to the third line, we use the constraint \eqref{eq:con}:
\begin{align*}
    \sum_{k=1}^K \rho_k\pi_k(a\mid s)h(k,s,a)=1. 
\end{align*}
This inequality becomes an equality when $h_k(s,a;p)=1/\pi_*$. In conclusion, we have 
\begin{align*}
    &\var\bracks{\frac{1}{n}\sum_{k=1}^{K}\sum_{j=1}^{n_k}h(k,S_{kj},A_{kj})\epol(A_{kj} \mid S_{kj})\{R_{kj}-g(S_{kj},A_{kj})\}+g(S_{kj},\epol) }\\
    &\geq \frac{1}{n}\E_{\pi_*}\bracks{ \frac{1}{\pi^2_*(a\mid s)}\{\epol(a\mid s)\}^2 \var[r \mid s,a]}+\frac{1}{n}\var_{p_{S}}[v(s)]. 
\end{align*}
and it becomes an equality when $g=q(s,a),h=1/\pi_*$.

\begin{remark}[Another Proof]
This theorem is also proved from semiparametric theory in \cref{sec:efficiency_detail}. 

Consider an estimator $\Gamma(\Dcal;h,g)$ by fixing $h$ and $g$ under $\cm_b$. ($\cm_b$ is the model where $\{\pi_k\}_{k=1}^K$ is known)  Then, $\Gamma(\Dcal;h,g)$ is an asymptotically linear estimator. Thus, the influence function of an asymptotically linear estimator belongs to the set of gradients of $J$ relative to $\cm_b$. The EIF has the smallest norm among this set of gradients. Thus, 
\begin{align*}
   \var[\Gamma(\Dcal;h,g)]\geq \var[\tilde \phi(\Dcal)]. 
\end{align*}


\end{remark}

\end{proof}

\begin{proof}[Proof of \cref{thm:property_feasible}]
We consider the case $K=2$ for simplicity. We also suppose that samples in each strata are uniformly distributed. We prove 
\begin{align}
    \hat J_{\bi}(\hat h,\hat g) &= 0.5\Gamma(\Lcal_1; \hat g^{(1)},\hat h^{(1)})+0.5\Gamma(\Lcal_2; \hat g^{(2)},\hat h^{(2)})\\
    & = 0.5\Gamma(\Lcal_1; g,h)+0.5\Gamma(\Lcal_2; g, h)+\op(n^{-1/2}). 
\end{align}
Then, the proof is immediately concluded from (stratified sampling) CLT. This is proved as follows.  

The first term is further expanded as follows:
\begin{align}
\Gamma(\Lcal_1; \hat g^{(1)},\hat h^{(1)}) &=\{\Gamma(\Lcal_1; \hat g^{(1)},\hat h^{(1)})-\E[\Gamma(\Lcal_1; \hat g^{(1)},\hat h^{(1)})\mid \cu_1 ]\}-\{ \Gamma(\Lcal_1; g, h)-\E[\Gamma(\Lcal_1; g, h)] \} \label{eq:db1_pseudo}  \\
      &+\E[\Gamma(\Lcal_1; \hat g^{(1)},\hat h^{(1)})\mid \cu_1 ]- \E[\Gamma(\Lcal_1; g, h) ] \label{eq:db2_pseudo} \\
      &+\Gamma(\Lcal_1; g, h). \nonumber  
\end{align}
First, \eqref{eq:db2_pseudo} is $0$ since 
\begin{align*}
    \E[\Gamma(\Lcal_1; \hat g^{(1)},\hat h^{(1)})\mid \cu_1 ]- \E[\Gamma(\Lcal_1; g, h)\mid \cu_1 ]=J-J=0. 
\end{align*}
Here, we use the constraint \eqref{eq:con} on $h$ and $\hat h$. Second, we show \eqref{eq:db1_pseudo} is $\op(n'^{-1/2})$. The conditional expectation of \eqref{eq:db1_pseudo} conditioning on $\cu_1$ is $0$. The conditional variance conditioning on $\cu_1$ is 
\begin{align*}
  &\var\bracks{\Gamma(\Lcal_1; \hat g^{(1)},\hat h^{(1)})-\Gamma(\Lcal_1; g, h)\mid \cu_1}\\
  &=\frac{1}{n'^{(1)}}\sum_{k=1}^{K}\rho_k \var_{\pi_k}[\hat h^{(1)}(k,s,a)\epol(a\mid s)\{r-\hat g^{(1)}(s,a)\}+\hat g^{(1)}(s,\epol)-\{h(k,s,a)\epol(a\mid s)\{r-g(s,a)\}\}-g(s,\epol) \mid \cu_1]. 
\end{align*}
We show that this term is $\op(n'^{-1})$. Then, the conditional Chebshev's inequality concludes that \eqref{eq:db2} is $\op(n'^{-1/2})$. To see this, what we have to show is that $\var_{\pi_k}[\cdot]$ is $\op(1)$ in the above. In fact, we have
\begin{align*}
&\var_{\pi_k}[\hat h^{(1)}(k,s,a)\epol(a\mid s)\{r-\hat g^{(1)}(s,a)\}+\hat g^{(1)}(s,\epol)-\{h(k,s,a)\epol(a\mid s)\{r-g(s,a)\}-g(s,\epol) \mid \cu_1]\\
&\leq 2\var_{\pi_k}[\{\hat h^{(1)}(k,s,a)-h(k,s,a)\}\epol(a\mid s)r \mid \cu_1]\\ &+2\var_{\pi_k}[-h(k,s,a)\epol(a\mid s)\hat g^{(1)}(s,a)+\hat g^{(1)}(s,\epol)+h_k(a,s)\epol(a\mid s)g(s,a)-g(s,\epol) \mid \cu_1] \\
&+2\var_{\pi_k}[\{\hat h^{(1)}(k,s,a)-h(k,s,a)\}\epol(a\mid s)\{\hat g^{(1)}(s,a)-g(s,a)\} \mid \cu_1]\\
&\lesssim \max(\|\hat h^{(1)}-h\|,\|\hat g^{(1)}-g\|)=\op(1). 
\end{align*}
Here, we use $\var_{\pi_k}[a+b]\leq 2\var_{\pi_k}[a]+2\var_{\pi_k}[b]$.

\end{proof}

\begin{proof}[Proof of \cref{thm:efficient,}]
Before checking this proof, refer to the proof of \cref{thm:b_known}. We use the result there. 

Since the model $\cm_b$ is smaller than the model $\cm$, the function $\tilde \phi(o)$ is still a gradient of $J$ w.r.t $\cm_b$. We show this gradient again lies in the tangent space. First, we calculate the tangent space. The tangent space of the model $\cm_{b}$ is 
\begin{align*}
    \braces{\sum_{k=1}^K\sum_{j=1}^{n_k}\braces{g_0(s_{kj})+g_{K+1}(s_{kj},a_{kj},r_{kj}) }\in L_2(o)}, 
\end{align*}
where 
\begin{align*}
\E[g_{0}(S_{kj})]=0, \, \E[ g_{K+1}(S_{kj},A_{kj},R_{kj})|S_{kj},A_{kj}]=0\,(1\leq k\leq K,1\leq j\leq n_k). 
\end{align*} 
The function $\tilde \phi(o)$ lies in the tangent space by taking $g_0(s)=v(s),g_{K+1}(s,a,r)=(r-q(s,a))$.

\end{proof}

\begin{proof}[Proof of \cref{thm:b_known}]

We follow the following steps.
\begin{enumerate}
    \item Calculate some gradient (a candidate of EIF) of the target functional $J$ w.r.t model $\cm$.
    \item Calculate the tangent space w.r.t the model $\cm$. 
    \item Show that a candidate of EIF in Step 1 lies in the tangent space. Then, this concludes that a candidate of EIF in Step 1 is actually the EIF.
\end{enumerate}

\paragraph{Calculation of the gradient}

We start with positing a parametric model: 
\begin{align*}
    p(o;\theta)=\prod_{k=1}^{K}\prod_{j=1}^{n_k}p_{S}(s_{kj};\theta_0)\pi_k(a_{kj}\mid s_{kj};\theta_k)p_{R\mid S,A}(r_{kj}\mid s_{kj},a_{kj};\theta_{K+1}). 
\end{align*}
We define the corresponding gradients:
\begin{align*}
    g_{S}(s)=\nabla_{\theta_0} \log p_{S}(s),\,g(k,s,a)=\nabla_{\theta_k} \log \pi_k(a\mid s),\,g_{R\mid S,A}=\nabla_{\theta_{K+1}} \log p_{R\mid S,A}(r\mid s,a). 
\end{align*}
To derive some gradient of the target functional w.r.t $\cm$, what we need is finding a function $f(o)$ satisfying 
\begin{align*}
    \nabla J(\theta)=\E[f(\Dcal)\nabla \log p(\Dcal;\theta) ]. 
\end{align*}
We take the gradient as follows: 
\begin{align*}
    \nabla J(\theta)&= \E_{\pi_*}\bracks{\frac{\epol}{\pi_*}r \braces{g_{S}(s)+g_{R\mid S,A}(s,a,r)} }\\
    &= \E_{\pi_*}\bracks{\psi(s,a,r)\braces{g_{S}(s)+g_{R\mid S,A}(s,a,r)} }, 
\end{align*}
where $\psi(s,a,r)=\epol/\pi_*(r-q(s,a))+v(s)-J$. This is equal to the following 
\begin{align*}
    \E\bracks{ \braces{\frac{1}{n}\sum_{k=1}^{K}\sum_{j=1}^{n_k}\phi(S_{kj},A_{kj},R_{kj})} \braces{\sum_{k=1}^{K}\sum_{j=1}^{n_k}g_{S}(S_{kj})+g_{k}(S_{kj},A_{kj})+g_{R|S.A}(S_{kj},A_{kj},    R_{kj}) }}
\end{align*}
since the above is equal to 
\begin{align*}
    &\frac{1}{n}\sum_{k=1}^{K}n_k\E_{\pi_k}\bracks{\psi(s,a,r)\{g_{S}(s)+g_{k}(s,a)+g_{R\mid S,A}(s,a,r) \} }\\
    &=\frac{1}{n}\sum_{k=1}^{K}n_k\E_{\pi_k}\bracks{\psi(s,a,r)\{g_{S}(s)+g_{R\mid S,A}(s,a,r) \} }\\
    &=\E_{\pi_*}\bracks{\psi(s,a,r)\{g_{S}(s)+g_{R\mid S,A}(s,a,r) \} }.  
\end{align*}
Thus, the following function 
\begin{align*}
    \tilde \phi(o)=\frac{1}{n}\sum_{k=1}^{K}\sum_{j=1}^{n_k}\psi(s_{kj},a_{kj},r_{kj}). 
\end{align*}
is a derivative of the target functional $J$ w.r.t the model $\cm$. 

\paragraph{Calculation of the tangent space }
Following a standard derivation way of the tangent space. \citep{TsiatisAnastasiosA2006STaM,VaartA.W.vander1998As}, the tangent space of the model $\cm$ is 
\begin{align*}
    \braces{\sum_{k=1}^K\sum_{j=1}^{n_k}\braces{g_0(s_{kj})+g_k(s_{kj},a_{kj})+g_{K+1}(s_{kj},a_{kj},r_{kj}) }\in L_2(o)}
\end{align*}
where 
\begin{align*}
\E[g_{0}(S_{kj})]=0, \E[g_k(S_{kj},A_{kj})|S_{kj}]=0  ,\,    \E[ g_{K+1}(S_{kj},A_{kj},R_{kj})|S_{kj},A_{kj}]=0\,(1\leq k\leq K,1\leq j\leq n_k). 
\end{align*}
and $L_2(o)$ is an $l_2$ space at the true density. 

\paragraph{Last part}

We can easily check that the $\tilde \phi(o)$ lies in the tangent space by taking $g_0(s)=v(s),g(k,s,a)=0,\,g_{K+1}(s,a,r)=(r-q(s,a))$. Thus, $\tilde \phi(o)$ is the EIF.

\end{proof}

\begin{proof}[Proof of \cref{thm:efficiency}]
In this proof, we define 
\begin{align*}
    \phi(s,a,r;\pi,g):=\epol/\pi\{r-g\}+g(s,\epol). 
\end{align*}
For simplicity, we consider the case where $K=2$ and assume that samples in each strata are uniformly distributed: 
\begin{align}\label{eq:multi}
     \hat J_{\dr} = 0.5\E_{n'^{(1)}}[\phi(s,a,r; \hat \pi^{(1)}_*,\hat q^{(1)})]+0.5\E_{n'^{(2)}}[\phi(s,a,r; \hat \pi^{(2)}_*,\hat q^{(2)})],  
\end{align}
where $\E_{{n'^{(i)}}}[\cdot]$ denotes an empirical approximation over the $i$-th fold data.
We prove 
\begin{align}\label{eq:final}
    \hat J_{\dr}  = 0.5\E_{n'^{(1)}}[\phi(s,a,r;  \pi_*, q)]+0.5\E_{n'^{(2)}}[\phi(s,a,r; \pi_*, q)]+\op(n'^{-1/2}).  
\end{align}
Then, the proof is immediately concluded from CLT.

The first term in \cref{eq:multi} is further expanded as follows:
\begin{align}
\E_{n'^{(1)}}[\phi(s,a,r; \hat \pi^{(1)}_*,\hat q^{(1)})] &=\frac{1}{\sqrt{n'^{(1)}}}\G_{n'^{(1)}}[\phi(s,a,r; \hat \pi^{(1)}_*,\hat q^{(1)} )-\phi(s,a,r;\pi_*,q)] \label{eq:db1}  \\
      &-\E_{\pi_*}[\phi(s,a,r; \pi_*,q)]+ \E_{\pi_*}[\phi(s,a,r;\hat \pi^{(1)}_*,\hat q^{(1)})\mid \cu_1] \label{eq:db2} \\
      &+ \E_{n'^{(1)}}[\phi(s,a,r; \pi_*,q)]. \nonumber  
\end{align}
Here, we define 
\begin{align*}
    &\frac{1}{\sqrt{n'^{(1)}}}\G_{n'^{(1)}}[\phi(s,a,r; \hat \pi^{(1)}_*,\hat q^{(1)} )-\phi(s,a,r;\pi_*,q)]\\ 
    &=\braces{\E_{n'^{(1)}}[\phi(s,a,r; \hat \pi^{(1)}_*,\hat q^{(1)} )-\phi(s,a,r;\pi_*,q)]-\E_{\pi_*}[\phi(s,a,r; \hat \pi^{(1)}_*,\hat q^{(1)} )-\phi(s,a,r;\pi_*,q)\mid \cu_1]  }
\end{align*}

First, we show \eqref{eq:db2} is $\op(n'^{-1/2})$. This is proved as 
\begin{align*}
    &\left|\E_{\pi_*}[\phi(s,a,r; \pi_*,q)]- \E_{\pi_*}[\phi(s,a,r;\hat \pi^{(1)}_*,\hat q^{(1)})\mid \cu_1]\right|\\
    &=\left|\E_{\pi_*}\bracks{(\epol/\pi_*-\epol/\hat \pi^{(1)}_*)(r-q) \mid \cu_1 }\right|+\left|\E_{\pi_*}\bracks{v-\hat v-\frac{\epol}{\pi_*}q+\frac{\epol}{\pi_*}\hat q\mid \cu_1}\right|\\ 
    &+\left|\E_{\pi_*}\bracks{\{\epol/\pi_*-\epol/\hat \pi^{(1)}_* \}\{-q+q^{(1)}\}\mid \cu_1}\right| \\ 
    &=|\E_{\pi_*}\bracks{\{\epol/\pi_*-\epol/\hat \pi^{(1)}_* \}\{-q+q^{(1)}\}\mid \cu_1}|\lesssim \|\hat \pi^{(1)}_*-\pi_*\|_2\|\hat q^{(1)}-q\|_2=\op(n'^{-1/2}). 
\end{align*}
Second, we show \eqref{eq:db1} is $\op(n'^{-1/2})$. The conditional expectation conditioning on $\cu_1$ is
\begin{align*}
&\E\bracks{\braces{\E_{n'^{(1)}}[\phi(s,a,r; \hat \pi^{(1)}_*,\hat q^{(1)} )-\phi(s,a,r;\pi_*,q)]}\mid \cu_1}-\E_{\pi_*}[\phi(s,a,r; \hat \pi^{(1)}_*,\hat q^{(1)} )-\phi(s,a,r;\pi_*,q)\mid \cu_1]  \\
&=0. 
\end{align*} 
The conditional variance conditioning on $\cu_1$ is 
\begin{align*}
  &\var\bracks{\braces{\E_{n'^{(1)}}[\phi(s,a,r; \hat \pi^{(1)}_*,\hat q^{(1)} )-\phi(s,a,r;\pi_*,q)]}\mid \cu_1}\\
  &=\frac{1}{n'^{(1)}}\sum_{k=1}^{K}\rho_k \var_{\pi_k}[ \phi(s,a,r; \hat \pi^{(1)}_*,\hat q^{(1)} )-\phi(s,a,r;\pi_*,q)\mid \cu_1] \\
  &\leq \frac{1}{n'^{(1)}}\var_{\pi_*}[ \phi(s,a,r; \hat \pi^{(1)}_*,\hat q^{(1)} )-\phi(s,a,r;\pi_*,q)\mid \cu_1] \\ 
  &\lesssim \frac{1}{n'^{(1)}}\max\{\|\hat \pi^{(1)}_*-\pi_*\|_2,\,\|\hat q^{(1)}-q\|_2 \}=\op(n'^{-1}).
\end{align*}
From the second line to the third line, we invoke \cref{thm:stratified}. Then, the conditional Chebshev's inequality concludes that \eqref{eq:db1} is $\op(n'^{-1/2})$. 

To summarize, \eqref{eq:db1} and \eqref{eq:db2} are $\op(n'^{-1/2})$. Thus, \cref{eq:final} is concluded. 

\end{proof}

\begin{proof}[Proof of \cref{thm:double_robustness}]
In this proof, we define 
\begin{align*}
    \phi(s,a,r;\pi,g):=\epol/\pi\{r-g\}+g(s,\epol). 
\end{align*}
For simplicity, we consider the case where $K=2$ and samples are uniformly distributed: 
\begin{align*}
     \hat J_{\dr} = 0.5\E_{n'^{(1)}}[\phi(s,a,r; \hat \pi^{(1)}_*,\hat q^{(1)})]+0.5\E_{n'^{(2)}}[\phi(s,a,r; \hat \pi^{(2)}_*,\hat q^{(2)})]. 
\end{align*}
where $\E_{{n'^{(i)}}}$ denotes an empirical approximation over the $i$-th fold data.

The first term is further expanded as follows:
\begin{align}
\E_{n'^{(1)}}[\phi(s,a,r; \hat \pi^{(1)}_*,\hat q^{(1)})]-J &=\frac{1}{\sqrt{n'^{(1)}}}\G_{n'^{(1)}}[\phi(s,a,r; \hat \pi^{(1)}_*,\hat q^{(1)} )-\phi(s,a,r;\pi^{\dagger}_*,q^{\dagger})] \label{eq:db1_1}  \\
      &-\E_{\pi_*}[\phi(s,a,r; \pi^{\dagger}_*,q^{\dagger})]+ \E_{\pi_*}[\phi(s,a,r;\hat \pi^{(1)}_*,\hat q^{(1)})\mid \cu_1] \label{eq:db2_1} \\
      &+ \E_{n'^{(1)}}[\phi(s,a,r; \pi^{\dagger}_*,q^{\dagger})]-J.\label{eq:db3_1}
\end{align}
As in the proof of \cref{thm:efficiency}, \cref{eq:db1_1} is $\op(n'^{-1/2}_1)$. The third term \eqref{eq:db3_1} is $\bigO_p(n'^{-1/2}_1)$ from CLT noting the mean is $0$ because 
\begin{align*}
    \E[\E_{n'^{(1)}}[\phi(s,a,r; \pi^{\dagger}_*,q^{\dagger})]]-J=0. 
\end{align*}
Here, we use the assumption that $\pi^{\dagger}_*$ or $q^{\dagger}$ is actually the true function. The second term is $\bigO_p(n'^{-1/2}_1)$ since 
\begin{align*}
    &|\E_{\pi_*}[\phi(s,a,r; \pi^{\dagger}_*,q^{\dagger})]- \E_{\pi_*}[\phi(s,a,r;\hat \pi^{(1)}_*,\hat q^{(1)})\mid \cu_1]|\\
    &\leq \left|\E_{\pi_*}\bracks{(\epol/\pi^{\dagger}_*-\epol/\hat \pi^{(1)}_*)(r-q^{\dagger})  \mid \cu_1}\right|+\left|\E_{\pi_*}\bracks{v^{\dagger}-\hat v-\frac{\epol}{\pi_*}q^{\dagger}+\frac{\epol}{\pi_*}\hat q\mid \cu_1}\right|\\ 
    &+\left|\E_{\pi_*}\bracks{\{\epol/\pi^{\dagger}_*-\epol/\hat \pi^{(1)}_* \}\{-q^{\dagger}+q^{(1)}\}\mid \cu_1}\right| \\ 
    &\lesssim \max(\|\hat \pi^{(1)}_*-\pi^{\dagger}_*\|_2\|,\hat q^{(1)}-q^{\dagger}\|_2)=\bigO_p(n'^{-1/2}_1). 
\end{align*}
In conclusion, $\E_{n'^{(1)}}[\phi(s,a,r; \hat \pi^{(1)}_*,\hat q^{(1)})]-J=\bigO_p(n'^{-1/2}_1)$. This concludes that $\hat J_{\dr}$ is $\sqrt{n'_1}$-consistent. 
\end{proof}

\begin{proof}[Proof of \cref{thm:stratified}]
Before stating the proof, in the DGP $\Dcal'$, we define $N_i,(1\leq i\leq K)$:
\begin{align*}
    (N_1,\cdots,N_K)\sim \mathrm{Multi}(n,\rho_1,\cdots,\rho_K),
\end{align*}
where $\E[N_k]=n_k$. Note that each $N_i$ is a random variable unlike a fixed constant $n_i$. 

First, we show that this estimator is unbiased. This is proved as follows:
\begin{align*}
    &\E_{\pi_{*}}[\phi(s,a,r;g)]\\
    &=\int \phi(s,a,r;g)I(\pi_{*}(a\mid s)>0)p_{R\mid S,A}(r\mid s,a)\pi_{*}(a\mid s)p_S(s)\mathrm{d}(s,a,r)=J. 
\end{align*}
Next, we show the inequality $  \var_{\Dcal'}[\P_n[\phi(s,a,r;g)]]\geq  \var_{\Dcal}[\P_n[\phi(s,a,r;g)]]$. From law of total variance, this is proved by 
\begin{align*}
   \var_{\Dcal'}[\P_n[\phi(s,a,r;g)]]&= \E[\var_{D'}[\P_n[\phi(s,a,r;g)]|\{N_k\}_{k=1}^K]]+\var[\E[\P_n[\phi(s,a,r;g)]|\{N_k\}_{k=1}^K|]] \\
   &\geq \E[\var_{D'}[\P_n[\phi(s,a,r;g)]|\{N_k\}_{k=1}^K]]=\E\bracks{\frac{N_k}{n^2}\sum_{k=1}^{K}\var_{\pi_k}[\phi(s,a,r;g)] } \\
   &= \E\bracks{\frac{\rho_k}{n}\sum_{k=1}^{K}\var_{\pi_k}[\phi(s,a,r;g)] }= \var_{\Dcal}[\P_n[\phi(s,a,r;g)]].
\end{align*}

We show the last statement. First, we have 
\begin{align*}
    \E[\P_n[\phi(s,a,r;g)]|\{N_k\}_{k=1}^K|] &=\frac{1}{n}\sum_{k=1}^{K}N_k\E_{\pi_k}\bracks{\frac{\epol}{\pi_*}\{r-g(s,a)\}+g(s,\epol)}\\
   &=\E[g(s,\epol)] +\frac{1}{n}\sum_{k=1}^{K}N_k\E_{\pi_k}\bracks{\frac{\epol}{\pi_*}\{r-g(s,a)\}}. 
\end{align*}
Then, when $g(s,a)=q(s,a)$, the equality $\var_{\Dcal'}[\P_n[\phi(s,a,r;g)]]= \var_{\Dcal}[\P_n[\phi(s,a,r;g)]]$ holds for any $\epol,\pi_{*}$ since 
\begin{align*}
    \var[\E[\P_n[\phi(s,a,r;g)]|\{N_k\}_{k=1}^K]]=0. 
\end{align*}
To get the equality $ \var_{\Dcal'}[\P_n[\phi(s,a,r;g)]]=  \var_{\Dcal}[\P_n[\phi(s,a,r;g)]]$ for any $\epol,\pi_{*}$, we need  
\begin{align*}
   \E_{\pi_k}\bracks{\frac{\epol}{\pi_*}\{r-g(s,a)\}}=0,\,\forall \pi_{*},\epol,1\leq \forall k\leq K.
\end{align*}
This implies 
\begin{align*}
    \E_{\pi_{*}}[\epol/\pi_{*}\{r-g(s,a)\}]=  \E_{\epol}[q(s,a)-g(s,a)]=0,\forall \epol,\pi_{*}.  
\end{align*}
This is only satisfied when $q(s,a)=g(s,a)$.

\begin{remark}
We can show the statement of the inequality regarding the variances by more direct calculation. We have 
\begin{align*}
\var_{\pi_{*}}\left[\phi(s,a,r;g)\right] &=\E_{\pi_{*}}\left[\left\{\phi(s,a,r;g)\right\}^2\right] -\E_{\epol}[r]^2,\\
\sum_{k=1}^K \frac{n_k}{n}\var_{\pi_k}\left[\phi(s,a,r;g)\right] &=\E_{\pi_{*}}\left[\left\{\phi(s,a,r;g)\right\}^2\right]-\sum_{k=1}^K \frac{n_k}{n}\E_{\pi_k}\left[\phi(s,a,r;g)\right]^2.
\end{align*}
Thus, the desired statement $ \var_{\Dcal'}[\P_n[\phi(s,a,r;g)]]\geq  \var_{\Dcal}[\P_n[\phi(s,a,r;g)]]$ for any $\epol,\pi_{*}$
is reduced to 
\begin{align*}
  \E_{\epol}[r]^2\leq \sum_{k=1}^K \frac{n_k}{n}\E_{\pi_k}\left[\phi(s,a,r;g)\right]^2. 
\end{align*}
This is proved by 
\begin{align*}
    \E_{\epol}\left[r\right]^2&=     \left\{\sum_{k=1}^K \frac{n_k}{n}\E_{\pi_k}\left[\phi(s,a,r;g)\right]\right\}^2= \sum_{k=1,j=1}^K \frac{n_kn_j}{n^2}\E_{\pi_k}\left[\phi(s,a,r;g)\right]\E_{\pi_j}\left[\phi(s,a,r;g)\right]\\
    &\leq \sum_{k=1,j=1}^K \frac{n_kn_j}{2n^2}\left\{\E_{\pi_k}\left[\phi(s,a,r;g)\right]^2+\E_{\pi_j}\left[\phi(s,a,r;g)\right]^2\right\}=\sum_{k=1}^K \frac{n_k}{n}\E_{\pi_k}\left[\phi(s,a,r;g)\right]^2.
\end{align*}
Here, we use $2ab\leq a^2+b^2$. 
\end{remark}

\end{proof}

\begin{proof}[Proof of \cref{thm:more_doubly}]
From the assumption, we have $\|\hat q^{(1)}_{\mo}-\check q\|=\op(1)$. We consider the case $K=2$ for simplicity:  
\begin{align*}
     \hat J_{\mo} = 0.5\E_{n'^{(1)}}[\phi(s,a,r; \hat q^{(1)}_{\mo})]+0.5\E_{n'^{(2)}}[\phi(s,a,r; \hat q^{(2)}_{\mo})]. 
\end{align*}
where $\E_{{n'^{(i)}}}$ denotes an empirical approximation over the $i$-th fold data. The first term is further expanded as follows:
\begin{align}
\E_{n'^{(1)}}[\phi(s,a,r; \hat q^{(1)}_{\mo})]-J &=\frac{1}{\sqrt{n'^{(1)}}}\G_{n'^{(1)}}[\phi(s,a,r;\hat q^{(1)}_{\mo} )-\phi(s,a,r;\check{q})] \label{eq:db1_2}  \\
      &-\E_{\pi_*}[\phi(s,a,r; \check q^{})]+\E_{\pi_*}[\phi(s,a,r;\hat q^{(1)}_{\mo})\mid \cu_1] \label{eq:db2_2} \\
      &+ \E_{n'^{(1)}}[\phi(s,a,r; \check q^{})]-J.\label{eq:db3_2}
\end{align}
As in the proof of \cref{thm:efficiency}, \cref{eq:db1_2} is $\op(n'^{-1/2}_1)$. The second term is $0$. In conclusion, $$\E_{n'^{(1)}}[\phi(s,a,r; \hat q^{(1)}_{\mo})]-J= \E_{n'^{(1)}}[\phi(s,a,r; \check q^{})]-J+o_p(n'^{-1/2}_1).$$
This concludes the statement:
\begin{align*}
     \hat J_{\mo}-J =\E_{n}[\phi(s,a,r; \check q^{})]-J+o_p(n'^{-1/2}_1). 
\end{align*}
\end{proof}

\begin{table}[h]
\centering
\caption{Dataset Statistics}
\begin{tabular}{c|cccc}
    \toprule
     Dataset Name &  OptDigits & SatImage & PenDigits & Letter  \\ \midrule \midrule
     \#Classes ($l$)   & 10  & 6  & 10 & 26  \\
     \#Data ($n$)  &  5620 &  6435 &  10992 & 20000 \\
    \bottomrule
\end{tabular}
\label{tab:dataset}
\end{table}

\section{Detailed Experimental Setup and Additional Results} \label{app:experiment}

\paragraph{Datasets.}
We use 4 datasets from the UCI Machine Learning Repository.\footnote{https://archive.ics.uci.edu/ml/index.php}
The dataset statistics are displayed in Table~\ref{tab:dataset}.

\paragraph{Detailed Experimental Procedure.}

A multi-class classification dataset consists of $(s_i,y_i)_{i=1}^{n}$ where $s_i \in \mathbb{R}^{d}$ is a context vector and $y_i \in \{1,\cdots,l\}$ is a class for an index $i$. The value $l$ is the number of class.
A classification algorithm assigning $s$ to $y$ is considered to be a policy from a context to an action where we regard $y$ as an action. 
When the prediction by the algorithm is correct, i.e., $y_i=\hat{y}_i$, we observe the unit reward $i$, otherwise the reward is $0$.
In this way, we can construct a contextual bandit dataset consisting of the set of triplets $\{(s_i,a_i,r_i)\}_{i=1}^{n}$ where $a_i:=\hat{y}_i$ and $r_i := \mathbb{I} \{y_i = \hat{y}_i\}$.

We summarize the whole experimental procedures below:
\begin{enumerate}
    \item We split the original data into training (30\%) and evaluation (70\%) sets.
    \item We train logistic regression using the training set to obtain a base deterministic policy $\pi_{\dett}$.
    \item Following Table~\ref{tab:policy_params}, we construct the logging and evaluation policies.
    \item We measure the accuracy of the evaluation policy and use it as its ground truth policy value.
    \item We regard $100 \times \rho_1/(1-\rho_1)=n_1/n_2$\% of the evaluation set as $\Dcal_1$ (generated by $\pi_1$) and the rest as $\Dcal_2$ (generated by $\pi_2$) where $\rho_1/(1-\rho_1)=n_1/n_2 \in \{0.1, 0.25, 0.5, 1, 2, 4, 10\}$. 
    A smaller value of $n_1/n_2$ leads to a larger data size of $\Dcal_2$ that is generated by a logging policy dissimilar to the evaluation policy.
    \item Using the evaluation set (consisting of $\Dcal_1$ and $\Dcal_2$), an estimator $\hat{J}$ estimates the policy value of the evaluation policy $J$.
\end{enumerate}

\begin{figure*}[t!]\begin{minipage}[b]{.5\textwidth}
    \centering
    \includegraphics[clip,width=\textwidth]{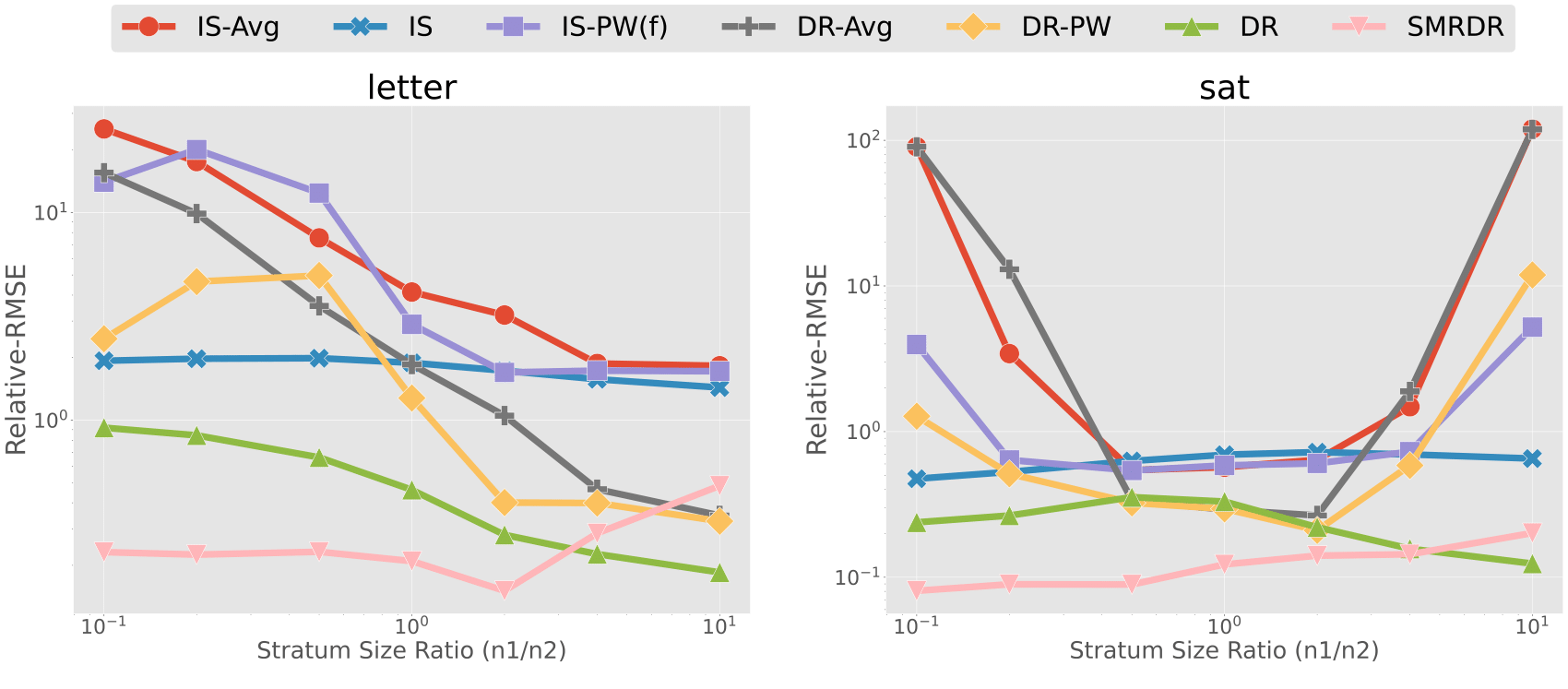}
    \caption{Comparing proposed estimators to some variants of IS type estimators.}
    \label{fig:dr_vs_is_app}
\end{minipage}
\begin{minipage}[b]{.5\textwidth}
    \includegraphics[clip,width=\textwidth]{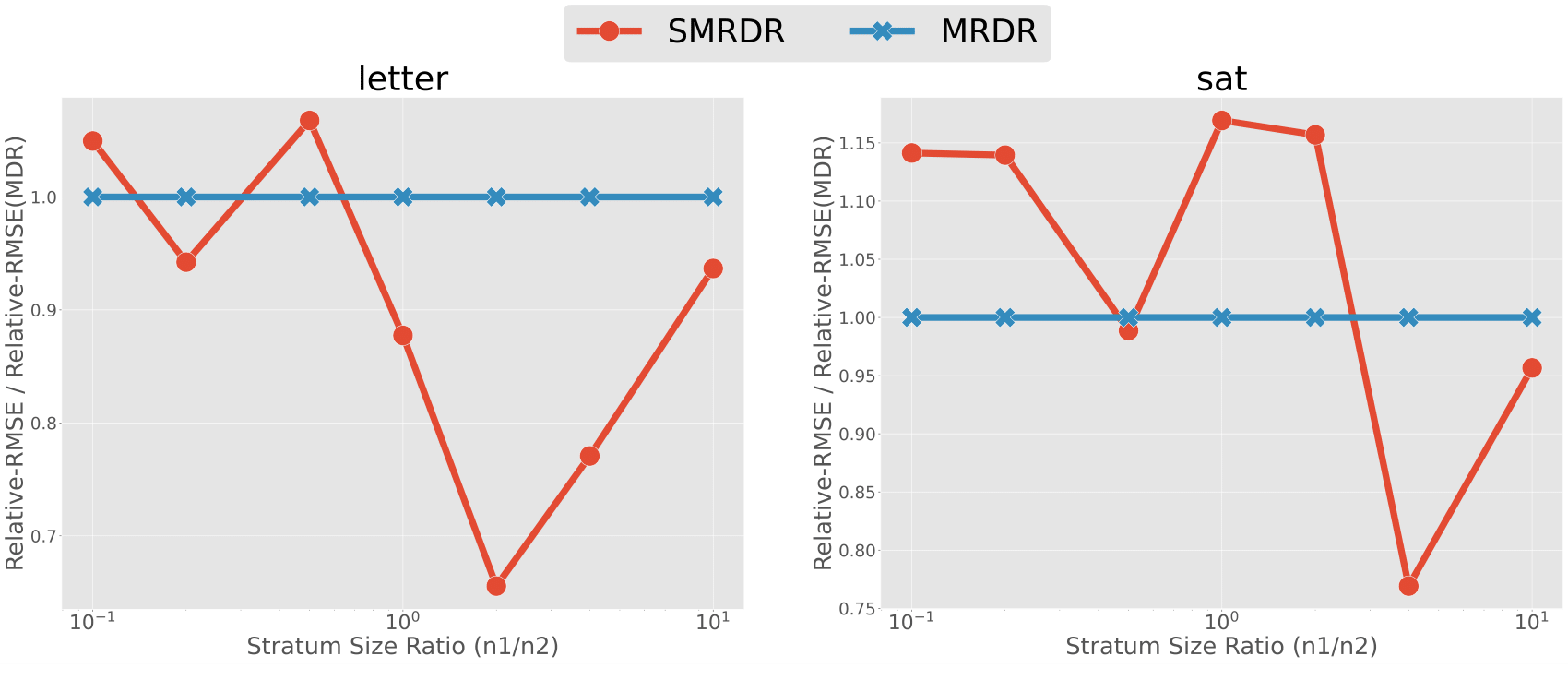}
    \caption{Comparing SMRDR (leveraging the stratification) and MRDR (ignoring the stratification).}
    \label{fig:smdr_vs_mdr_app}
\end{minipage}\end{figure*}

\paragraph{Additional Results.} 
Figure~\ref{fig:dr_vs_is_app} and~\ref{fig:smdr_vs_mdr_app} show the results on the same experiment as conducted in Section~\ref{sec:experiment} in the main text on the SatImage and Letter datasets.

\end{document}